\documentclass[10pt]{article}
\usepackage[utf8]{inputenc} 
\usepackage[T1]{fontenc}    
\usepackage{lmodern}
\usepackage{url}            
\usepackage{booktabs}       
\usepackage{amsfonts}       
\usepackage{microtype}      
\usepackage{algorithm}
\usepackage{algorithmic}
\usepackage{amsmath, amsthm, amssymb, multirow}
\usepackage{graphicx}
\usepackage{wrapfig}
\usepackage[body={6.5in,8in}]{geometry}
\usepackage{epsfig}

\usepackage{xcolor}
\usepackage{soul}
\usepackage{subfigure}
\usepackage{array}
\usepackage{epstopdf}
\usepackage{amsbsy,bm,ifthen,color}

\usepackage{cite}
\usepackage{mathrsfs}
\usepackage{subfig}

\newtheorem{theo}{Theorem}
\newtheorem{definition}{Definition}
\newtheorem{lemma}{Lemma}
\newtheorem{corollary}{Corollary}
\newcommand{\argmin}{\mathop{\arg\min}}
\newcommand{\argmax}{\mathop{\arg\max}}

\title{A Generalized Unbiased Risk Estimator for Learning with Augmented Classes}

\author{
  Senlin Shu$^1$, Shuo He$^{2}$, Haobo Wang$^{3}$, Hongxin Wei$^{4}$, Tao Xiang$^{1}$, Lei Feng$^{4}$\thanks{Corresponding author: Lei Feng $<$lfengqaq@gmail.com$>$.}\\
  $^1$Chongqing University, China\\
  $^2$University of Electronic Science and Technology of China, China\\
  $^3$Zhejiang University, China \\
  $^4$Nanyang Technological University, Singapore\\
}
\begin{document}
\maketitle





\begin{abstract}
In contrast to the standard learning paradigm where all classes can be observed in training data, \emph{learning with augmented classes} (LAC) tackles the problem where augmented classes unobserved in the training data may emerge in the test phase. Previous research showed that given unlabeled data, an \emph{unbiased risk estimator} (URE) can be derived, which can be minimized for LAC with theoretical guarantees. However, this URE is only restricted to the specific type of one-versus-rest loss functions for multi-class classification, making it not flexible enough when the loss needs to be changed with the dataset in practice. In this paper, we propose a \emph{generalized} URE that can be equipped with arbitrary loss functions while maintaining the theoretical guarantees, given unlabeled data for LAC. To alleviate the issue of negative empirical risk commonly encountered by previous studies, we further propose a novel risk-penalty regularization term. Experiments demonstrate the effectiveness of our proposed method.

\end{abstract}

\section{Introduction}
Machine learning approaches have achieved great performance on a variety of tasks, and most of them focus on the stationary learning environment. However, the learning environment in many real-world scenarios could be open and change gradually, which requires the learning approaches to have the ability of handling the distribution change in the non-stationary environment \cite{sugiyama2012machine,gama2014survey,zhao2019distribution,zhang2020unbiased}.

This paper considers a specific problem where the class distribution changes from the training phase to the test phase, called \emph{learning with augmented classes} (LAC). In LAC, some augmented classes unobserved in the training phase might emerge in the test phase. In order to make accurate and reliable predictions, the learning model is required to distinguish augmented classes and keep good generalization performance over the test distribution.

The major difficulty in LAC is how to exploit the relationships between known and augmented classes. To overcome this difficulty, various learning methods have been proposed. For example, by learning a compact geometric description of known classes to distinguish augmented classes that are far away from the description, the anomaly detection or novelty detection methods can be used (e.g., iForest \cite{liu2008isolation}, one-class SVM \cite{scholkopf2001estimating,tax2004support}, and kernel density estimation \cite{parzen1962estimation,kim2012robust}). By exploiting unlabeled data with the low-density separation assumption to adjust the classification decision boundary \cite{da2014learning}, the performance of LAC can be empirically improved.

A recent study \cite{zhang2020unbiased} showed that by exploiting unlabeled data, an \emph{unbiased risk estimator} (URE) for LAC over the test distribution can be constructed under the \emph{class shift condition}. The motivation of this study stems from that although the instances from augmented classes cannot be observed in labeled data, their distribution information may be contained in unlabeled data and estimated by differentiating the distribution of known classes from unlabeled data. Such an URE is advantageous for the learning task, since it can lead to a theoretically grounded method based on empirical risk minimization. However, the URE for LAC derived by the previous study \cite{zhang2020unbiased} is only restricted to the specific type of one-versus-rest loss functions, making it not flexible enough when the loss needs to be changed with the dataset in practice.

To address the above issue, we make the following contributions in this paper:
\begin{itemize}
\item we propose a \emph{generalized} unbiased risk estimator that can be equipped with \emph{arbitrary} loss functions while maintaining the theoretical guarantees, given unlabeled data for learning with augmented classes.
\item We provide a theoretical analysis on the estimation error bound, which guarantees that the convergence of the empirical risk minimizer to the true risk minimizer.
\item We propose a novel risk-penalty regularization term that can be used to alleviate the issue of negative empirical risk commonly encountered by previous studies.
\item We conduct extensive experiments with different base models on widely used benchmark datasets to demonstrate the effectiveness of our proposed method.
\end{itemize}


\section{Preliminaries}
In this section, we introduce some preliminary knowledge. 
\paragraph{Learning with augmented classes.} 
In the training phase of LAC, we are given a labeled training set $\mathcal{D}_{\mathrm{L}}=\{\bm{x}_i,y_i\}_{i=1}^{n}$ independently sampled from a distribution of known classes $P_{\mathrm{kc}}$ over $\mathcal{X}\times\mathcal{Y}^\prime$, where $\mathcal{X}$ is the feature space and $\mathcal{Y}^\prime=\{1,2,\ldots,k\}$ is the label space of $k$ known classes. However, in the test phase, we need to predict unlabeled data sampled from the test distribution $P_{\mathrm{te}}$, where augmented classes unobserved in the training phase may emerge. Since the specific partition of augmented classes is unknown, they are normally predicted as a single class. In this way, the label space of the test distribution $P_{\mathrm{te}}$ is $\mathcal{Y}=\{1,2,\ldots,k,\mathtt{ac}\}$ where $\mathtt{ac}$ denotes the class grouped by augmented classes.
The \emph{class shift condition} \cite{zhang2020unbiased} was introduced to describe the relationship between the distributions of known and augmented classes:
\begin{gather}
\label{data_distribution}
P_{\mathrm{te}} = \theta\cdot P_{\mathrm{kc}} + (1-\theta)\cdot P_{\mathrm{ac}},
\end{gather}
where $\theta\in[0,1]$ indicates the mixture proportion. In addition, a set of unlabeled data $\mathcal{D}_{\mathrm{U}}=\{\bm{x}_i\}_{i=1}^m$ sampled from the test distribution is provided and will be used for model training. This learning scenario of available unlabeled data is conceivable since we can easily collect a large amount of unlabeled data from current environments. Such unlabeled data can be used to enrich the information of unobservable classes in the training phase, despite no revealed supervision information on the unlabeled data.
\paragraph{Unbiased risk estimator for LAC.} Under the class shift condition in LAC, the expected classification risk over the test distribution can be formulated as
\begin{align}
\nonumber
R(f)&=\mathbb{E}_{(\bm{x},y)\sim P_{\mathrm{te}}}[\mathcal{L}(f(\bm{x}),y)]\\
&=\theta\mathbb{E}_{(\bm{x},y)\sim P_{\mathrm{kc}}}[\mathcal{L}(f(\bm{x}),y)]
 + (1-\theta)\mathbb{E}_{(\bm{x},y=\mathtt{ac})\sim P_{\mathrm{ac}}}[\mathcal{L}(f(\bm{x}),\mathtt{ac})].
 \label{classification_risk}
\end{align}
The recent study \cite{zhang2020unbiased} only considered the \emph{one-versus-rest} (OVR) loss function as the classification loss in Eq.~(\ref{classification_risk}), which takes the following form for $k$-class classification:
\begin{gather}
\nonumber
\mathcal{L}_{\mathrm{OVR}}(f(\bm{x}),y) = \psi(f_{y}(\bm{x})) + \sum\nolimits_{i=1,i\neq y}^k \psi(-f_i(\bm{x})),
\end{gather}
where $\psi(\cdot):\mathbb{R}\mapsto\mathbb{R}_+$ denotes a binary loss function. By substituting the OVR loss $\mathcal{L}_{\mathrm{OVR}}$ into the classification risk, we have $R^{\mathrm{OVR}}(f)=\mathbb{E}_{(\bm{x},y)\sim P_{\mathrm{te}}}[\mathcal{L}_{\mathrm{OVR}}(f(\bm{x}),y)]$.
The previous study \cite{zhang2020unbiased} showed that an equivalent expression of $R^{\mathrm{OVR}}(f)$ in the setting of LAC with unlabeled data drawn from the test distribution can be derived:
\begin{align}
\nonumber
R_{\mathrm{LAC}}^{\mathrm{OVR}}(f)&=\theta\cdot\mathbb{E}_{p_{\mathrm{kc}}(\bm{x},y)}[\psi(f_y(\bm{x}))-\psi(-f_y(\bm{x})) + \psi(-f_{\mathtt{ac}}(\bm{x})) - \psi(f_{\mathtt{ac}}(\bm{x}))] \\
\nonumber
&\qquad\qquad\qquad\qquad\qquad\qquad\qquad\qquad\qquad\qquad+ \mathbb{E}_{\bm{x}\sim P_{\mathrm{te}}}[\psi(f_{\mathrm{ac}}(\bm{x}))+\sum\nolimits_{i=1}^k\psi(-f_k(\bm{x}))].
\end{align}
As can be verified, $R^{\mathrm{OVR}}_{\mathrm{LAC}}(f)=R^{\mathrm{OVR}}(f)$, and thus its empirical version $\widehat{R}^{\mathrm{OVR}}_{\mathrm{LAC}}(f)$ is an \emph{unbiased risk estimator} (URE). As we can see, this URE is only restricted to the OVR loss for multi-class classification, making it not flexible enough when the loss needs to be changed with the dataset in practice.

\section{The Proposed Method}
In this section, we first propose a generalized URE and provide a theoretical analysis for the derived URE. Then, we propose a novel risk-penalty regularization term that can be used to alleviate the issue of negative empirical risk.
\subsection{A Generalized Unbiased Risk Estimator}
\begin{theo}
\label{thm1}
Under the class shift condition in Eq.~(\ref{data_distribution}), the classification risk $R(f)$ can be equivalently expressed as:
\begin{align}
\label{LAC_expected_risk}
R_{\mathrm{LAC}}(f) 
&= \theta \mathbb{E}_{(\bm{x},y)\sim P_{\mathrm{kc}}}[\mathcal{L}(f(\boldsymbol{x}),y)-\mathcal{L}(f(\boldsymbol{x}),\mathtt{ac})]
 + \mathbb{E}_{\bm{x}\sim P_{\mathrm{te}}}[\mathcal{L}(f(\boldsymbol{x}),\mathtt{ac})]=R(f).
\end{align}
\end{theo}
The proof of Theorem \ref{thm1} is provided in Appendix \ref{proof_theorem_1}. 
Theorem \ref{thm1} shows that we are able to learn a classifier from labeled data from the distribution of known classes and unlabeled data from the test distribution, without any restrictions on the loss function $\mathcal{L}$ for multi-class classification.
\begin{corollary}
\label{coro1}
If the one-versus-rest (OVR) loss $\mathcal{L}_{\mathrm{OVR}}$ is used as the classification loss in our derived risk $R_{\mathrm{LAC}}(f)$, then we can exactly recover the risk $R_{\mathrm{LAC}}^{\mathrm{OVR}}(f)$ derived by the previous study \cite{zhang2020unbiased}.
\end{corollary}
The proof of Corollary \ref{coro1} is omitted here, since it is quite straightforward
to verify that by directly inserting $\mathcal{L}_{\mathrm{OVR}}$ into $R_{\mathrm{LAC}}(f)$. Corollary \ref{coro1} shows that our proposed URE is a generalization of the URE proposed by Zhang et al. \cite{zhang2020unbiased} and can be compatible with arbitrary loss functions.

Given a set of labeled data $\mathcal{D}_{\mathrm{L}}=\{\bm{x}_i,y_i\}_{i=1}^{n}$ drawn from the distribution of known classes $P_{\mathrm{kc}}$ and a set of unlabeled data $\mathcal{D}_{\mathrm{U}}=\{\bm{x}_i\}_{i=1}^m$ drawn from the test distribution $P_\mathrm{te}$, we can obtain the following URE, which is the empirical approximation of ${R}_{\mathrm{LAC}}$:
\begin{align}
\label{LAC_empirical risk}
\widehat{R}_{\mathrm{LAC}}(f) 
&= \frac{\theta}{n}\sum\nolimits_{i=1}^n\Big(\mathcal{L}(f(\boldsymbol{x}_i),y_i)-\mathcal{L}(f(\boldsymbol{x}_i),\mathtt{ac})\Big) 
 + \frac{1}{m}\sum\nolimits_{j=1}^m\mathcal{L}(f(\boldsymbol{x}_j),\mathtt{ac}).
\end{align}
In this way, we are able to learn an effective multi-class classifier by directly minimizing $\widehat{R}_{\mathrm{LAC}}(f)$. Since there are no restrictions on the loss function $\mathcal{L}$ and the model $f$, we can use any loss and any model for LAC with unlabeled data.
\subsection{Theoretical Analysis}
Here, we establish an estimation error bound for our proposed URE based on the widely used \emph{Rademacher complexity} \cite{bartlett2002rademacher}.
\begin{definition}[Rademacher Complexity]
Let $n$ be a positive integer, $\bm{x}_1,\ldots,\bm{x}_n$ be independent and identically distributed random variables drawn from a probability distribution with density $\mu$, $\mathcal{G}={g:\mathcal{X}\mapsto\mathbb{R}}$ be a class of measurable functions, and $\boldsymbol{\sigma}=(\sigma_1,\ldots,\sigma_n)$ be Rademacher variables that take value from $\{+1,-1\}$ with even probabilities. Then, the Rademacher complexity of $\mathcal{G}$ is defined as
\begin{gather}
\nonumber
\mathfrak{R}_{n}(\mathcal{G}):=\mathbb{E}_{\bm{x}_1,\ldots,\bm{x}_n\overset{\mathrm{i.i.d.}}{\sim}\mu}\mathbb{E}_{\boldsymbol{\sigma}}\Big[\sup\nolimits_{g\in\mathcal{G}}\frac{1}{n}\sum\nolimits_{i=1}^n\sigma_i g(\bm{x}_i)\Big].
\end{gather}
\end{definition}
Let us represent $\mathcal{F}$ by $\mathcal{F}=\mathcal{F}_1\times\mathcal{F}_2\cdots\mathcal{F}_{k+1}$ where $\mathcal{F}_i=\{\bm{x}\mapsto f_i(\bm{x}) \mid f\in\mathcal{F}\}$ and $\mathcal{F}_{k+1}$ is specially for the augmented class $\mathtt{ac}$. Thus we can denote by $\mathfrak{R}_{n}(\mathcal{F}_y)$ the Rademacher complexity of $\mathcal{F}_y$ for the $y$-th class, given the sample size $n$ over the test distribution $P_{\mathrm{te}}$. It is commonly assumed that $\mathfrak{R}_{n}(\mathcal{F}_y)\leq C_{\mathcal{F}}/\sqrt{n}$ for all $y\in\mathcal{Y}$ \cite{bao2018classification}, where $C_{\mathcal{F}}$ is a positive constant. Let the learned classifier by our URE be $\widehat{f}=\argmin_{f\in\mathcal{F}}\widehat{R}_{\mathrm{LAC}}(f)$ and the optimal classifier learned by minimizing the true classification risk be $f^\star=\argmin_{f\in\mathcal{F}}R(f)$. Then we have the following theorem.
\begin{theo}
\label{thm2}
Assume the multi-class loss $\mathcal{L}(f(\bm{x}),y)$ is $\rho$-Lipschitz ($0<\rho<\infty$) with respect to $f(\bm{x})$ for all $y\in\mathcal{Y}$ and upper bounded by a positive constant $C_{\mathcal{L}}$, i.e., $C_{\mathcal{L}} = \sup_{\bm{x}\in\mathcal{X},y\in\mathcal{Y},f\in\mathcal{F}}\mathcal{L}(f(\bm{x}),y)$. Then, for any $\delta>0$, with probability at least $1-\delta$,
\begin{align}
\nonumber
R(\widehat{f})-R(f^\star) &\leq C_{k,\rho,\delta}(\frac{2\theta}{\sqrt{n}}+\frac{1}{\sqrt{m}}),
\end{align}
where $C_{k,\rho,\delta} = (4\sqrt{2}\rho(k+1)C_{\mathcal{F}}+2C_{\mathcal{L}}\sqrt{\frac{\log\frac{4}{\delta}}{2}})$.
\end{theo}
The proof of Theorem \ref{thm2} is provided in Appendix \ref{proof_theorem_2}. 
Theorem \ref{thm2} shows that our proposed method is consistent, i.e., $R(\widehat{f})\rightarrow R(f^\star)$ as $n\rightarrow\infty$ and $m\rightarrow\infty$. The convergence rate is $\mathcal{O}_p(1/\sqrt{n}+1/\sqrt{m})$, where $\mathcal{O}_p$ denotes the order in probability. This order is the optimal parametric rate for
the empirical risk minimization without additional assumptions (Mendelson, 2008).
\subsection{Risk-Penalty Regularization}
As shown in Eq.~(\ref{LAC_empirical risk}), there exist negative terms in the URE $\widehat{R}_{\mathrm{LAC}}(f)$. Since the widely used classification loss (i.e., the cross entropy loss) is unbounded above, $\widehat{R}_{\mathrm{LAC}}(f)$ could be unbounded below, which causes the overfitting issue, and this issue becomes especially terrible when deep models are employed. Such an issue was encountered and demonstrated by many previous studies \cite{kiryo2017positive,lu2020mitigating,zhang2020unbiased,feng2021pointwise}. To alleviate this issue, these studies normally wrapped the terms that could make empirical risk negative by certain correction
functions, such as the \emph{rectified linear unit} (ReLU) function $g(z) = \max(0, z)$ and the \emph{absolute value} (ABS) function $g(z) = |z|$. Despite the effectiveness of the correction functions used by previous studies \cite{lu2020mitigating}, using a correction function could make the empirical risk not unbiased anymore. Motivated by this, we aim to keep the original URE unchanged and impose a regularization term to alleviate the issue of negative empirical risk.

We note that the above issue in our work comes from the use of equality: $(1-\theta)\mathbb{E}_{(\bm{x},y=\mathtt{ac})\sim P_{\mathrm{ac}}}[\mathcal{L}(f(\bm{x}),\mathtt{ac})]
=\mathbb{E}_{(\bm{x},y)\sim P_{\mathrm{te}}}[\mathcal{L}(f(\bm{x}),\mathtt{ac})]-\theta\mathbb{E}_{(\bm{x},y)\sim P_{\mathrm{kc}}}[\mathcal{L}(f(\bm{x}),\mathtt{ac})]$. 
Let us denote by $\widehat{R}_{\mathrm{PAC}}$ the empirical version of the expected risk on the left hand side of the above equality:
\begin{gather}
\nonumber
\widehat{R}_{\mathrm{PAC}}(f) =  \frac{1}{m}\sum_{j=1}^m\mathcal{L}(f(\boldsymbol{x}_j),\mathtt{ac}) - \frac{\theta}{n}\sum_{i=1}^n \mathcal{L}(f(\boldsymbol{x}_i),\mathtt{ac}).
\end{gather}
Our goal is to prevent the whole training objective from going to be unbounded below in the training process. Therefore, we choose to add an regularization term on the training objective, which could incur a cost when $\widehat{R}_{\mathrm{PAC}}(f)$ goes to be negative. In this way, the whole training objective would not be unbounded below, because there is always a positive cost incurred if $\widehat{R}_{\mathrm{PAC}}(f)$ goes to be negative.
Our proposed regularization term is presented as follows:
\begin{align}
\Omega(f) = \left\{\begin{matrix}
 (-\widehat{R}_{\mathrm{PAC}}(f))^{t}, &\mathrm{if}\ \widehat{R}_{\mathrm{PAC}}(f)<0, \\
0, &\mathrm{otherwise}, \\
\end{matrix}\right.
\end{align}
where $t\geq 0$ is a hyper-parameter. We plot the above regularization function in Figure \ref{fig1}. Using this regularization, our final training objective becomes $\widehat{R}_{\mathrm{LAC}}(f) + \lambda \Omega(f)$, where $\widehat{R}_{\mathrm{LAC}}(f)$ is our derived URE defined in Eq.~(\ref{LAC_empirical risk}) and $\lambda$ is a hyper-parameter that controls the importance of the regularization term. 
As can be seen from Figure 1, when the empirical risk becomes negative, an extra loss will be generated, thereby alleviating the issue of the negative empirical risk during training. 
Additionally, the widely used correction functions ReLU and ABS can be considered as spacial cases of our proposed negative risk penalty regularization. Specifically, when $t=1$ and $\lambda=1$, we obtain a formulation by using ReLU as the correction function. When $t=1$ and $\lambda=2$, we obtain a formulation by using ABS as the correction function. Therefore, our proposed risk-penalty regularization could be considered as a generalized solution to the issue of negative empirical risk.
\begin{figure}[!t]
    \centering
    \includegraphics[width=2.4in]{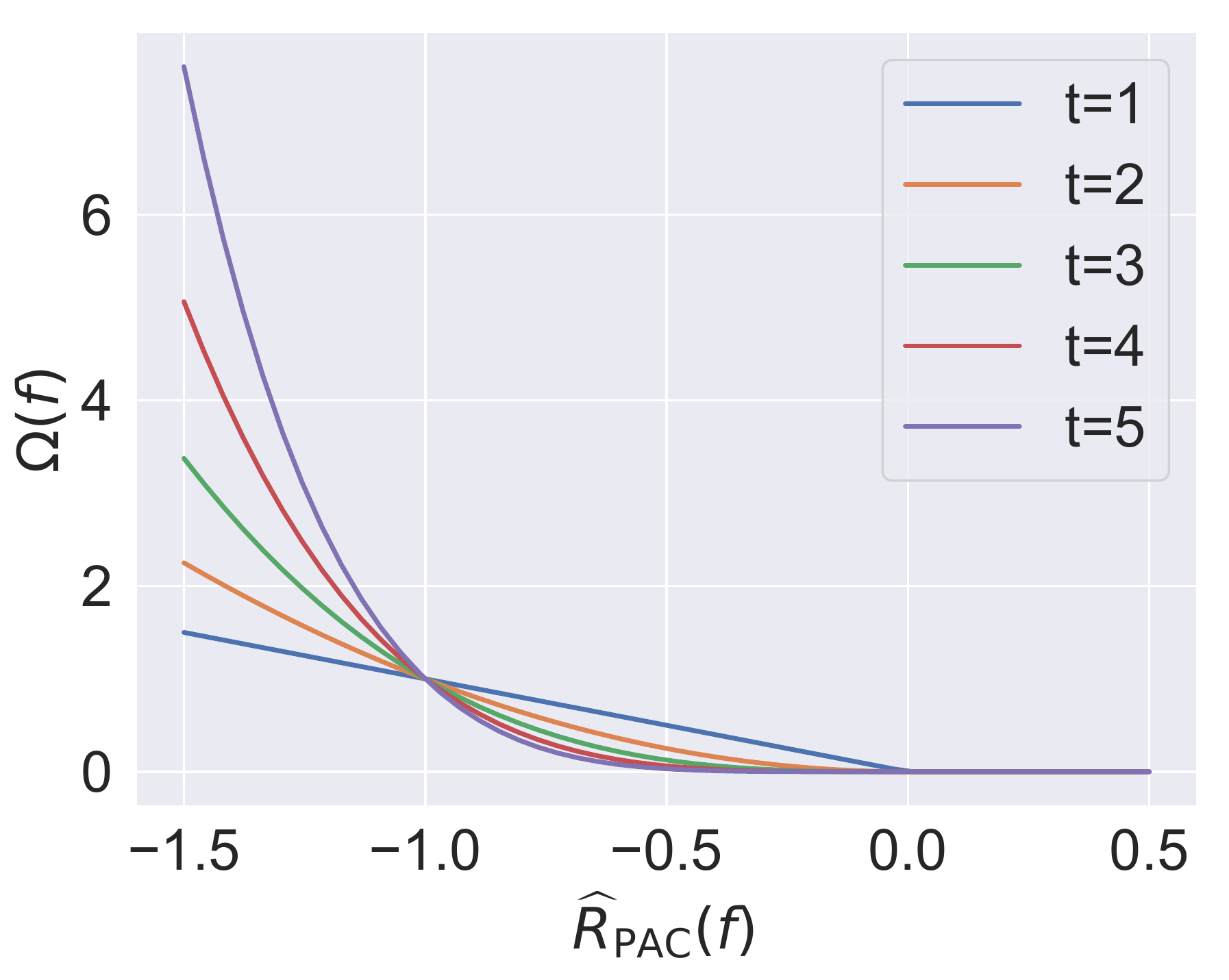}
    \caption{The function curve of our proposed risk-penalty regularization.}
    \label{fig1}
\end{figure}

\subsection{Handling Class Prior Shift}
The class shift condition in Eq.~(\ref{data_distribution}) describes the relationship between the distributions of known classes and augmented classes, with the assumption that the distribution of known classes keeps identical in both the training phase and the test phase. However, the learning environment could dynamically change in real-world scenarios, and thus the prior of each class could also vary from the training phase and the test phase. As shown by Zhang et al. \cite{zhang2020unbiased}, the following generalized class shift condition can be obtained:
\begin{gather}
\label{generalized_data_distribution}
P_{\mathrm{te}} = \sum\nolimits_{i=1}^k\theta_{\mathrm{te}}^i\cdot P_{i} + (1-\sum\nolimits_{i=1}^k\theta_{\mathrm{te}}^i)\cdot P_{\mathrm{ac}},
\end{gather}
where $P_i$ denotes the data distribution of the $i$-th known class and $\theta_{\mathrm{te}}^i$ denotes the class prior of the $i$-th known class in the test phase. Given the generalized class shift condition in Eq.~(\ref{generalized_data_distribution}), we can still obtain an equivalent representation of the classification risk $R(f)$.
\begin{theo}
\label{thm3}
Under the generalized class shift condition in Eq.~(\ref{generalized_data_distribution}), the classification risk $R(f)$ can be equivalently expressed as follows:
\begin{align}
\label{generalized_LAC_expected_risk}
&R_{\mathrm{LAC}}^{\mathrm{shift}}(f) 
= \sum\nolimits_{i=1}^k\theta_{\mathrm{te}}^i\cdot\mathbb{E}_{(\bm{x},y)\sim P_{i}}[\mathcal{L}(f(\boldsymbol{x}),y)
-\mathcal{L}(f(\boldsymbol{x}),\mathtt{ac})] + \mathbb{E}_{\bm{x}\sim P_{\mathrm{te}}}[\mathcal{L}(f(\boldsymbol{x}),\mathtt{ac})]=R(f).
\end{align}
\end{theo}
We omitted the proof of Theorem \ref{thm3}, since it is quite straightforward by following the proof of Theorem \ref{thm1}. Theorem \ref{thm3} shows that we can learn a classifier for LAC with unlabeled data under the class prior shift condition of known classes by empirically minimizing $R_{\mathrm{LAC}}^{\mathrm{shift}}(f)$.
\section{Experiments}
In this section, we conduct extensive experiments to evaluate the performance of our proposed method on various datasets using different models.
\subsection{Expermental Setup}
\paragraph{Datasets.}
We use six regular-scale datasets downloaded from the UCI Machine Learning Repository\footnote{\url{https://archive.ics.uci.edu}} \cite{dua2017uci}, including Har, Msplice, Normal, Optdigits, Texture, and Usps. Since they are not large-scale datasets, we train a linear model on these datasets. We also use four widely large-scale benchmark datasets, including MNIST\footnote{\url{http://yann.lecun.com/exdb/mnist/}} \cite{lecun1998gradient}, Fashion-MNIST\footnote{\url{https://github.com/zalandoresearch/fashion-mnist}} \cite{xiao2017fashion}, Kuzushiji-MNIST\footnote{http://codh.rois.ac.jp/kmnist/} \cite{clanuwat2018deep}, and SVHN\footnote{\url{http://ufldl.stanford.edu/housenumbers/}} \cite{netzer2011reading}. For MINST, Fashion-MNIST, and Kuzushiji-MNIST, we train a multilayer perceptron (MLP) model with three layers ($d-500-k$) and the ReLU activation function is used.
For the SVHN dataset, we train a VGG16 model \cite{simonyan2015}. The brief characteristics of all the used datasets are reported in Table \ref{dataset}. For each regular-scale dataset, half of the classes are selected as augmented classes and the remaining classes are considered as known classes. Besides, the number of labeled, unlabeled, and test examples set to 500, 1000, and 1000, respectively. For large-scale datasets, we select six classes as known classes and other classes are regarded as augmented classes. For MNIST, Fashion-MNIST, and Kuzushiji-MNIST, the number of the labeled, unlabeled, and test examples is set to 24000 (4000 per known class), 10000 (1000 per class), and 1000 (100 per class), respectively. For SVHN, the number of the labeled, unlabeled, and test examples is set to 24000 (4000 per known class), 25000 (2500 per class), and 1000 (100 per class), respectively. To comprehensively evaluate the performance of our proposed method, we report the mean value with standard deviation over 10 trials, in terms of the evaluation metrics including accuracy, Macro-F1, and AUC.
\begin{table}[!t]
	\centering
	\resizebox{0.6\textwidth}{!}{
		\setlength{\tabcolsep}{5.5mm}{
			\begin{tabular}{cccc}
				\toprule
				Dataset   &\# Examples & \# Features  &\# Classes \\
				\midrule
				Har &10,699  &561  &6 \\
				 Msplice &3,175  &240  &3\\
				 Normal &7,000  &2  &7 \\
					Optdigits &5,620   &62  &10   \\
					Texture &5,500   &40  &11   \\
					Usps &9,298 &256 &10  \\
				\midrule
					MNIST & 70,000  & 784 & 10  \\
					Fashion & 70,000  & 784 & 10  \\
					Kuzushiji & 70,000  & 784 & 10  \\
					SVHN & 99,268  & 3,072 & 10  \\
				\bottomrule
			\end{tabular}
		}
	}
 \caption{Brief characteristics of the used datasets.}
\label{dataset}
\end{table}
\begin{table*}[!t]
    \centering
    \resizebox{1.00\textwidth}{!}{
		\setlength{\tabcolsep}{6.0mm}{\begin{tabular}{c|cccccc}
    \toprule
    \multirow{2}{*}{Datasets} & \multicolumn{6}{c}{Accuracy} \\ \cmidrule(r){2-7}
 &OVR &EVM &LACU &iForest &EULAC &NRPR \\
\midrule
Har	&0.452$\pm$0.012 	&0.413$\pm$0.007 	&0.542$\pm$0.077 	&0.958$\pm$0.005 	&0.967$\pm$0.009 	&\textbf{0.986$\pm$0.007} 	\\
Msplice	&0.725$\pm$0.008 	&0.658$\pm$0.005 	&0.664$\pm$0.015 	&0.716$\pm$0.007 	&0.781$\pm$0.032 	&\textbf{0.785$\pm$0.020} 	\\
Normal	&0.570$\pm$0.001 	&0.567$\pm$0.003 	&0.567$\pm$0.050 	&0.775$\pm$0.036 	&0.755$\pm$0.051 	&\textbf{0.851$\pm$0.009} 	\\
Optdigits	&0.533$\pm$0.014 	&0.460$\pm$0.009 	&0.724$\pm$0.068 	&0.728$\pm$0.020 	&0.791$\pm$0.047 	&\textbf{0.863$\pm$0.037} 	\\
Texture	&0.570$\pm$0.010 	&0.474$\pm$0.018 	&0.520$\pm$0.062 	&0.601$\pm$0.029 	&0.732$\pm$0.109 	&\textbf{0.930$\pm$0.009} 	\\
Usps	&0.631$\pm$0.026 	&0.547$\pm$0.007 	&0.572$\pm$0.057 	&0.551$\pm$0.005 	&0.874$\pm$0.023 	&\textbf{0.878$\pm$0.012} 	\\
\midrule
\multirow{2}{*}{Datasets} & \multicolumn{6}{c}{Macro-F1} \\ \cmidrule(r){2-7}
 &OVR &EVM &LACU &iForest &EULAC &NRPR \\
\midrule
Har	&0.526$\pm$0.014 	&0.530$\pm$0.008 	&0.277$\pm$0.020 	&0.950$\pm$0.006 	&0.950$\pm$0.011 	&\textbf{0.980$\pm$0.009} 	\\
Msplice	&0.549$\pm$0.007 	&0.490$\pm$0.005 	&0.537$\pm$0.049 	&0.537$\pm$0.007 	&0.759$\pm$0.034 	&\textbf{0.767$\pm$0.021} 	\\
Normal	&0.679$\pm$0.001 	&0.658$\pm$0.028 	&0.405$\pm$0.221 	&0.747$\pm$0.050 	&0.586$\pm$0.103 	&\textbf{0.762$\pm$0.013} 	\\
Optdigits	&0.606$\pm$0.012 	&0.572$\pm$0.014 	&0.745$\pm$0.069 	&0.708$\pm$0.025 	&0.690$\pm$0.103 	&\textbf{0.859$\pm$0.026} 	\\
Texture	&0.709$\pm$0.008 	&0.561$\pm$0.026 	&0.294$\pm$0.194 	&0.696$\pm$0.019 	&0.540$\pm$0.192 	&\textbf{0.927$\pm$0.011} 	\\
Usps	&0.667$\pm$0.026 	&0.579$\pm$0.008 	&0.370$\pm$0.111 	&0.590$\pm$0.010 	&0.853$\pm$0.038 	&\textbf{0.867$\pm$0.012} 	\\
\midrule
\multirow{2}{*}{Datasets} & \multicolumn{6}{c}{AUC} \\ \cmidrule(r){2-7}
 &OVR &EVM &LACU &iForest &EULAC &NRPR \\
\midrule
Har	&0.789$\pm$0.006 	&0.754$\pm$0.007 	&0.695$\pm$0.051 	&0.953$\pm$0.005 	&0.998$\pm$0.001 	&\textbf{0.999$\pm$0.001} 	\\
Msplice	&0.737$\pm$0.007 	&0.652$\pm$0.006 	&0.748$\pm$0.011 	&0.733$\pm$0.006 	&0.924$\pm$0.013 	&\textbf{0.917$\pm$0.010} 	\\
Normal	&0.849$\pm$0.001 	&0.847$\pm$0.002 	&0.748$\pm$0.029 	&0.803$\pm$0.031 	&0.976$\pm$0.026 	&\textbf{0.999$\pm$0.000} 	\\
Optdigits	&0.866$\pm$0.005 	&0.836$\pm$0.008 	&0.834$\pm$0.041 	&0.798$\pm$0.021 	&0.907$\pm$0.053 	&\textbf{0.987$\pm$0.002} 	\\
Texture	&0.894$\pm$0.004 	&0.831$\pm$0.016 	&0.720$\pm$0.036 	&0.828$\pm$0.012 	&0.756$\pm$0.126 	&\textbf{0.995$\pm$0.001} 	\\
Usps	&0.870$\pm$0.007 	&0.846$\pm$0.007 	&0.743$\pm$0.034 	&0.843$\pm$0.004 	&0.979$\pm$0.005 	&\textbf{0.981$\pm$0.002} 	\\
\bottomrule
    \end{tabular}}}
    \caption{Test performance (mean$\pm$std) of each method on UCI datasets. The best performance is highlighted in bold.}
    \label{uci_results}
\end{table*}
\begin{table*}[!t]
    \centering
    \resizebox{1.00\textwidth}{!}{
		\setlength{\tabcolsep}{6.0mm}{\begin{tabular}{c|cccccc}
    \toprule
    \multirow{2}{*}{Datasets} & \multicolumn{6}{c}{Accuracy} \\ \cmidrule(r){2-7}
 &OVR &Softmax &Softmax-T &Openmax &EULAC &NRPR \\
\midrule
MNIST	&0.743$\pm$0.018 	&0.592$\pm$0.002 	&0.883$\pm$0.009 	&0.882$\pm$0.007 	&0.952$\pm$0.006 	&\textbf{0.960$\pm$0.006} 	\\
Fashion	&0.587$\pm$0.009 	&0.559$\pm$0.007 	&0.595$\pm$0.011 	&0.601$\pm$0.011 	&0.840$\pm$0.034 	&\textbf{0.875$\pm$0.010} 	\\
Kuzushiji	&0.696$\pm$0.017 	&0.578$\pm$0.004 	&0.828$\pm$0.009 	&0.835$\pm$0.012 	&0.881$\pm$0.016 	&\textbf{0.927$\pm$0.009} 	\\
SVHN	&0.595$\pm$0.022 	&0.564$\pm$0.006 	&0.720$\pm$0.037 	&0.786$\pm$0.013 	&0.835$\pm$0.021 	&\textbf{0.873$\pm$0.012} 	\\
\midrule
\multirow{2}{*}{Datasets} & \multicolumn{6}{c}{Macro-F1} \\ \cmidrule(r){2-7}
 &OVR &Softmax &Softmax-T &Openmax &EULAC &NRPR \\
\midrule
MNIST	&0.794$\pm$0.015 	&0.653$\pm$0.004 	&0.898$\pm$0.007 	&0.898$\pm$0.006 	&0.951$\pm$0.007 	&\textbf{0.960$\pm$0.006} 	\\
Fashion	&0.667$\pm$0.010 	&0.637$\pm$0.007 	&0.685$\pm$0.009 	&0.688$\pm$0.008 	&0.813$\pm$0.070 	&\textbf{0.872$\pm$0.010} 	\\
Kuzushiji	&0.746$\pm$0.014 	&0.630$\pm$0.005 	&0.847$\pm$0.009 	&0.851$\pm$0.011 	&0.877$\pm$0.021 	&\textbf{0.926$\pm$0.008} 	\\
SVHN	&0.640$\pm$0.020 	&0.614$\pm$0.007 	&0.760$\pm$0.029 	&0.801$\pm$0.014 	&0.831$\pm$0.021 	&\textbf{0.876$\pm$0.011} 	\\
\midrule
\multirow{2}{*}{Datasets} & \multicolumn{6}{c}{AUC} \\ \cmidrule(r){2-7}
 &OVR &Softmax &Softmax-T &Openmax &EULAC &NRPR \\
\midrule
MNIST	&0.929$\pm$0.005 	&0.890$\pm$0.002 	&0.939$\pm$0.005 	&0.939$\pm$0.005 	&0.996$\pm$0.002 	&\textbf{0.998$\pm$0.001} 	\\
Fashion	&0.868$\pm$0.004 	&0.864$\pm$0.005 	&0.822$\pm$0.005 	&0.820$\pm$0.004 	&0.953$\pm$0.039 	&\textbf{0.985$\pm$0.002} 	\\
Kuzushiji	&0.908$\pm$0.005 	&0.880$\pm$0.003 	&0.909$\pm$0.008 	&0.905$\pm$0.008 	&0.981$\pm$0.009 	&\textbf{0.994$\pm$0.001} 	\\
SVHN	&0.835$\pm$0.011 	&0.868$\pm$0.005 	&0.889$\pm$0.012 	&0.865$\pm$0.018 	&0.931$\pm$0.019 	&\textbf{0.984$\pm$0.003} 	\\
\bottomrule
    \end{tabular}}}
    \caption{Test performance (mean$\pm$std) of each method on benchmark datasets. The best performance is highlighted in bold.}
    \label{image_results}    
\end{table*}
\paragraph{Methods.}
For the experiments on regular-scale datasets, we compare with the following five methods:
\begin{itemize}
\item \textbf{OVR} \cite{rifkin2004defense}, which uses the one-versus-reset loss for training a multi-class classifier. For the LAC problem, OVR predicts the augmented class if $\mathrm{max}_{y\in \mathcal{Y}^\prime} f_{y}(\boldsymbol{x})<0$, otherwise it makes normal predictions $\argmax_{y\in \mathcal{Y}} f_{y}$ in known classes.
\item \textbf{EVM} \cite{rudd2017extreme}, which uses the extreme value machine to perform nonlinear kernel-free variable bandwidth incremental learning.
\item \textbf{LACU} \cite{da2014learning}, which exploits unlabeled data with the low-density separation assumption to adjust the classification decision boundary.
\item \textbf{iForest} \cite{liu2008isolation}, which first uses iForest to detect augmented classes and then uses the OVR method \cite{rifkin2004defense} to make normal predictions for known classes.
\item \textbf{EULAC} \cite{zhang2020unbiased}, which derives an unbiased risk estimator for LAC with unlabeled data.
\end{itemize}
For the experiments on large-scale datasets, we compare with OVR, EULAC, and following three methods:
\begin{itemize}
\item \textbf{Softmax}, which directly trains a multi-class classifier by using the softmax cross entropy loss and only makes normal predictions for known classes.
\item \textbf{Softmax-T}, which denotes the Softmax method with a threshold set for predicting the augmented class. It predicts the augmented class if $\mathrm{max}_{y\in \mathcal{Y}^\prime} g_{y}(\boldsymbol{x})$ is less than a given threshold, where $g_{y}(\boldsymbol{x})$ is a softmax score of the class $y$ for the instance $\bm{x}$. The threshold is set to 0.95.
\item \textbf{Openmax} \cite{bendale2016towards}, which can be considered as a calibrated version of the Softmax-T method. In Openmax, Weibull calibration is implemented as an augment to the SoftMax method while replacing the Softmax layer with a new OpenMax layer.
\end{itemize}
For the compared methods, we adopt the hyper-parameter settings suggested by respective papers. For our proposed method, we utilize the \emph{generalized cross entropy} (GCE) loss~\cite{zhang2018generalized} as the multi-class loss function, since it is a generalized version of the widely used cross entropy loss. We use the Adam optimization method~\cite{kingma2015adam}, with the number of training epochs set to 1500 on regular-scale datasets and 200 on large-scale datasets respectively. We set the mini-batch size to 500 on large-scale datasets and use the full batch size on regular-scale datasets. For regular-scale datasets, learning rate and weight decay are selected in $\{10^{-2},10^{-3},10^{-4}\}$. For large-scale datasets, learning rate and weight decay are selected in $\{10^{-3},10^{-4},10^{-5}\}$. For our method, $t$ and $\lambda$ are selected in $\{1,2,3\}$ and $\{0.2,0.4,\ldots,1.8,2\}$, respectively. Since our method uses a URE with negave-risk penalty regularization, we name it \textbf{NRPR}. All the experiments are conducted on GeForce GTX 3090 GPUs. 
\begin{table}[!t]
    \centering
    \resizebox{0.7\textwidth}{!}{
    \setlength{\tabcolsep}{10mm}{
    \begin{tabular}{c|ccccc}
    \toprule
\multirow{2}{*}{Datasets} & \multicolumn{3}{c}{Accuracy} \\
    \cmidrule(r){2-4}
   &ReLU &ABS &NRPR \\
    \midrule
MNIST	&0.865$\pm$0.034 	&0.776$\pm$0.014 	&\textbf{0.960$\pm$0.006} 	\\
Fashion	&0.809$\pm$0.021 	&0.747$\pm$0.019 	&\textbf{0.870$\pm$0.009} 	\\
Kuzushiji	&0.853$\pm$0.029 	&0.757$\pm$0.027 	&\textbf{0.927$\pm$0.009} 	\\
SVHN	&0.734$\pm$0.018 	&0.695$\pm$0.007 	&\textbf{0.865$\pm$0.010} 	\\
    \midrule
\multirow{2}{*}{Datasets}   & \multicolumn{3}{c}{Macro-F1} \\ \cmidrule(r){2-4}
   &ReLU &ABS &NRPR \\
    \midrule
MNIST	&0.883$\pm$0.027 	&0.813$\pm$0.011 	&\textbf{0.960$\pm$0.006} 	\\
Fashion	&0.825$\pm$0.017 	&0.782$\pm$0.015 	&\textbf{0.868$\pm$0.009} 	\\
Kuzushiji	&0.868$\pm$0.024 	&0.790$\pm$0.023 	&\textbf{0.926$\pm$0.008} 	\\
SVHN	&0.767$\pm$0.015 	&0.736$\pm$0.006 	&\textbf{0.867$\pm$0.011} 	\\
    \midrule
\multirow{2}{*}{Datasets}   & \multicolumn{3}{c}{AUC} \\ \cmidrule(r){2-4}
   &ReLU &ABS &NRPR \\
    \midrule
MNIST	&0.997$\pm$0.001 	&0.996$\pm$0.001 	&\textbf{0.998$\pm$0.001} 	\\
Fashion	&0.974$\pm$0.001 	&0.966$\pm$0.002 	&\textbf{0.985$\pm$0.002} 	\\
Kuzushiji	&0.992$\pm$0.002 	&0.984$\pm$0.004 	&\textbf{0.994$\pm$0.001} 	\\
SVHN	&\textbf{0.983$\pm$0.003} 	&0.980$\pm$0.002 	&0.981$\pm$0.003 	\\
\bottomrule
    \end{tabular}
    }
    }
    \caption{Performance comparison between our proposed risk-penalty regularization and risk-correction functions.}
    \label{reg_comparison}    
\end{table}
\begin{figure}[!htb]
    \centering
    \subfigure{\includegraphics[width=1.5in]{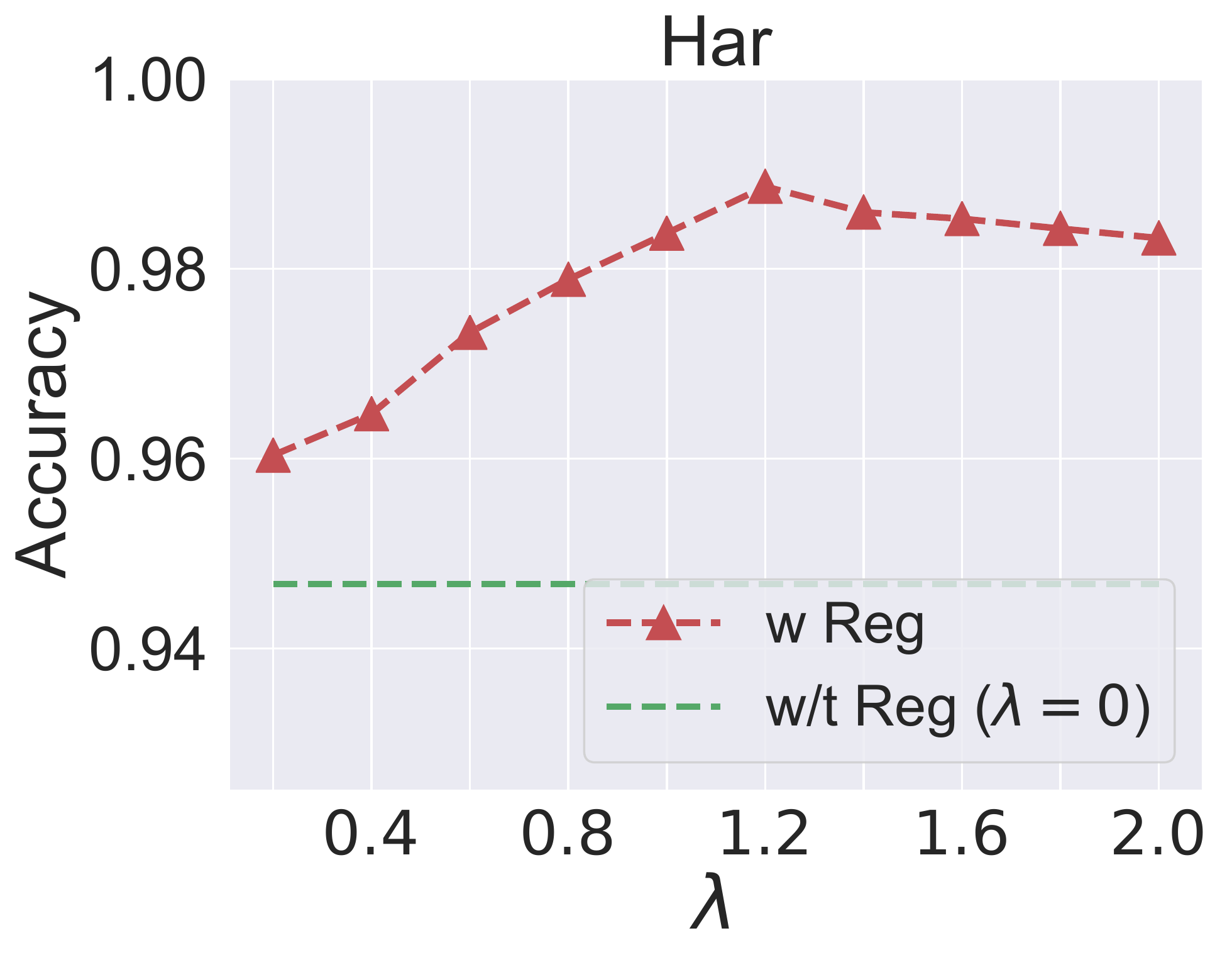}}
    \subfigure{\includegraphics[width=1.5in]{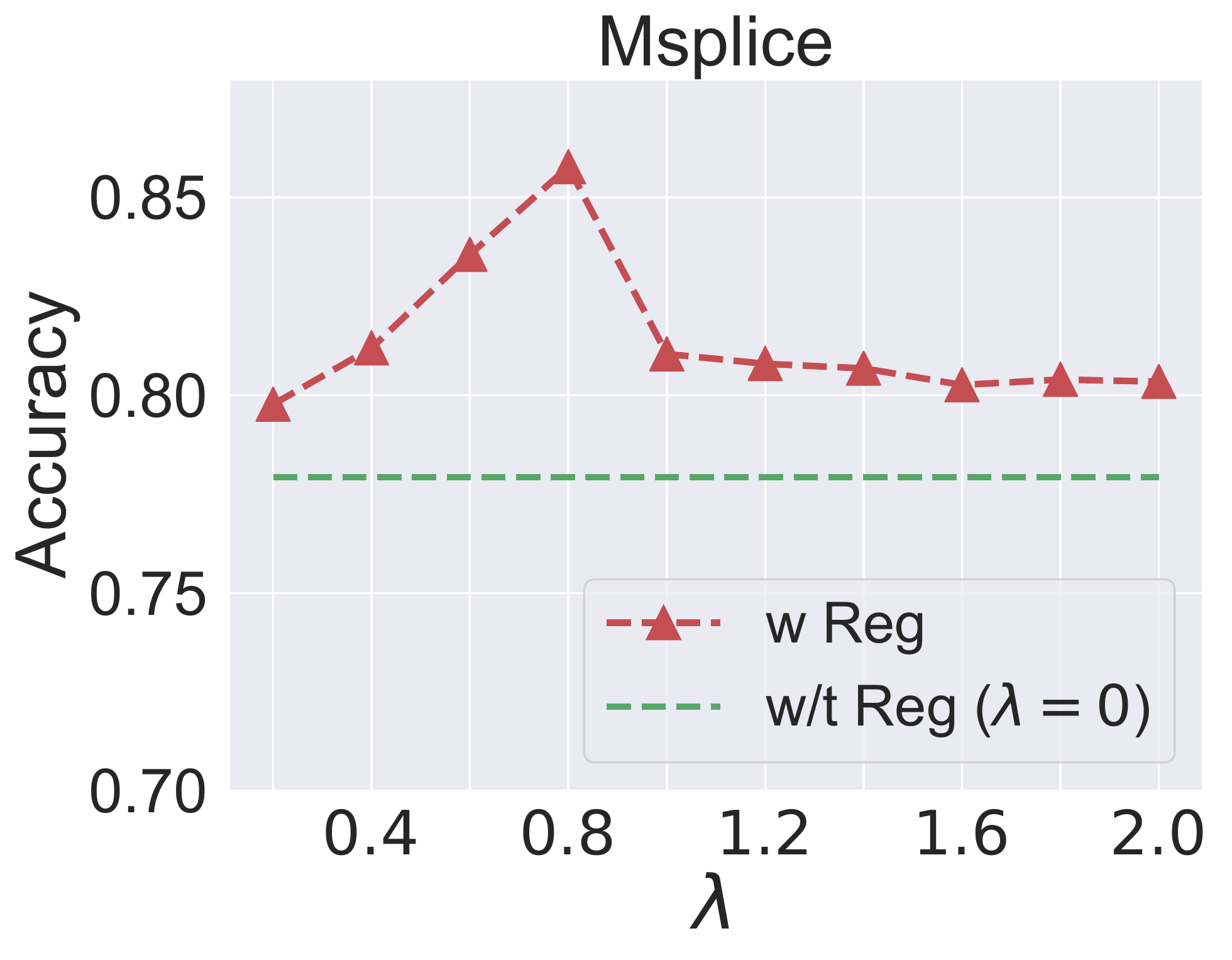}}
     \subfigure{\includegraphics[width=1.5in]{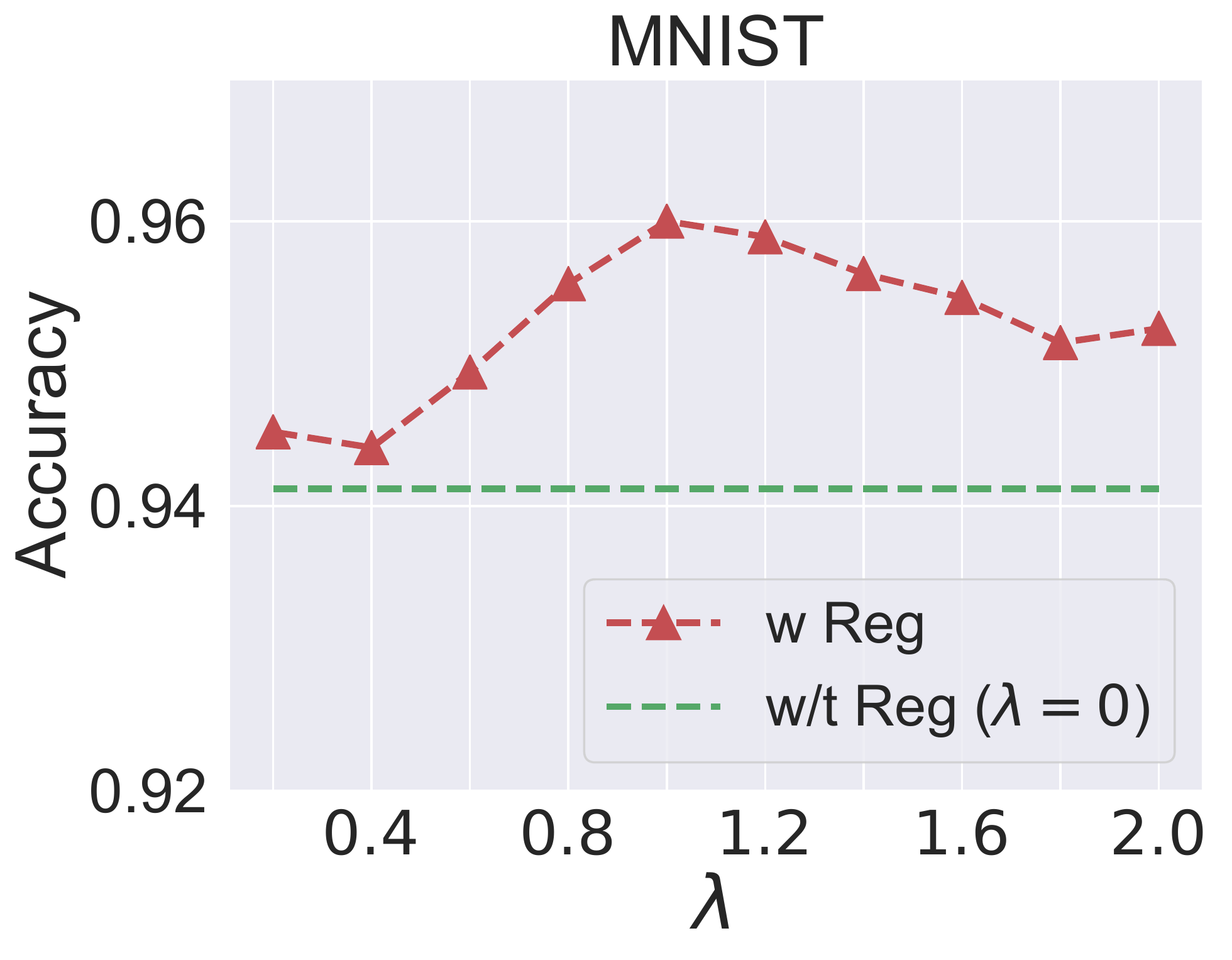}}
    \subfigure{\includegraphics[width=1.5in]{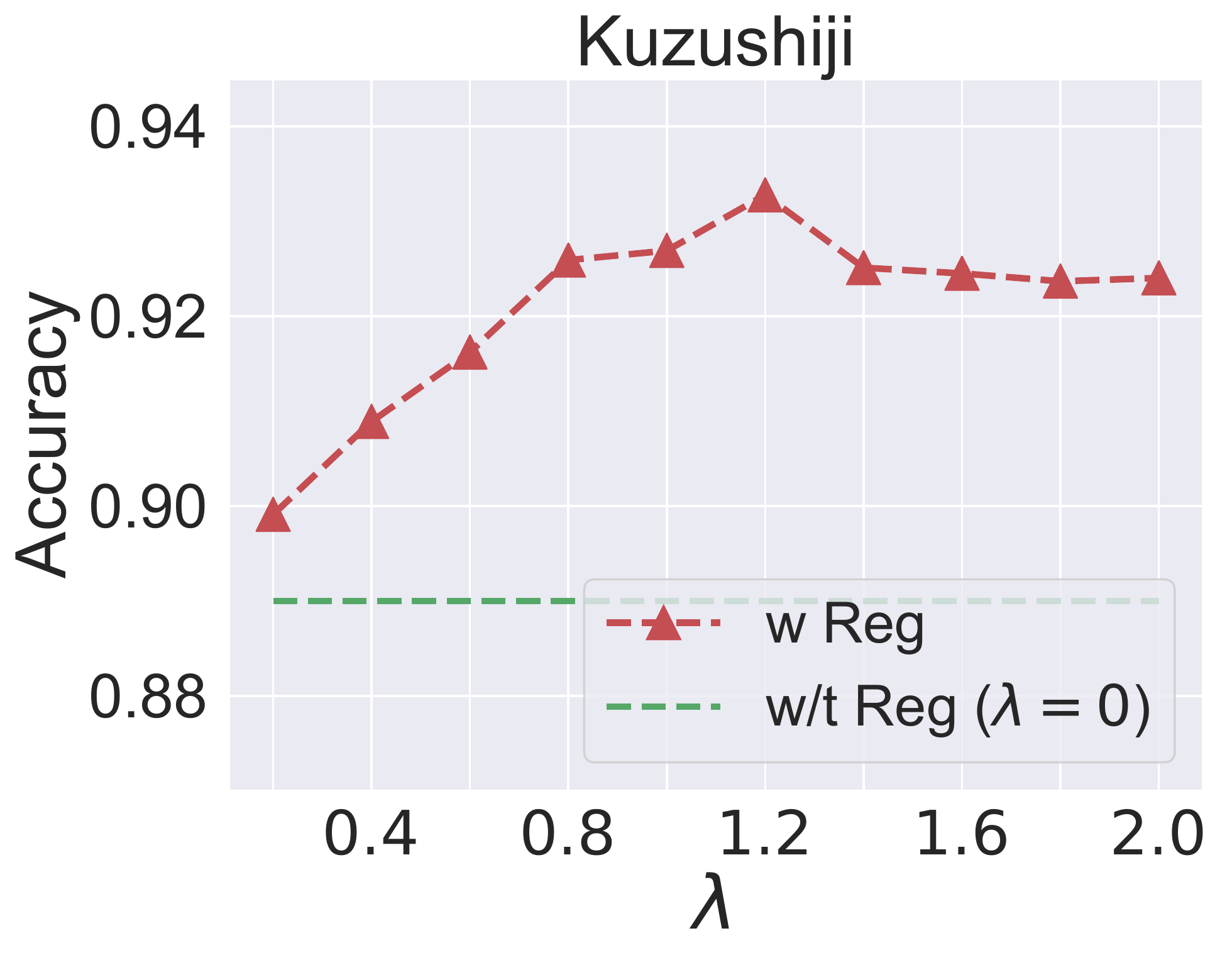}}
    \caption{Classification accuracy with different values of the regularization parameter $\lambda$.}
    \label{parameter_analysis}
\end{figure}
\subsection{Experimental Results}
\paragraph{Results on regular-scale datasets.}
Table \ref{uci_results} records the accuracy, Macro-F1, and AUC of each method on the
six regular-scale UCI datasets trained with a linear model. From Table~\ref{uci_results}, we can observe that our proposed method achieves the best performance in all the cases, which validates the effectiveness of our proposed method. Besides, the superiority of our method (using a generalized URE) over the EULAC method (using a URE relying on the OVR loss) is also clearly demonstrated, which indicates that our method is not only more flexible on the choice of used loss functions but also is more effective by substituting other loss functions rather than the OVR loss.
\paragraph{Results on large-scale datasets.}
Table \ref{image_results} records the accuracy, Macro-F1, and AUC of each method on the
four large-scale benchmark datasets. As can be observed from Table~\ref{image_results}, our proposed method also achieves the best performance in all the cases. In addition, the performance gap between the EULAC method and our proposed method is also considerably evident. These positive experimental results also support our proposed method.

\subsection{Further Analysis}
\paragraph{Effectiveness of the risk-penalty regularization.} In order to verify the effectiveness of our proposed risk-penalty regularization, we further compare it with the widely used risk-correction functions \cite{kiryo2017positive,lu2020mitigating} including the \emph{rectified linear unit} (ReLU) function $g(z) = \max(0, z)$ and the \emph{absolute value} (ABS) function $g(z) = |z|$. We conduct experiments on the four large-scale benchmark datasets and report the accuracy, Macro-F1, and AUC of each method in Table \ref{reg_comparison}. As can be seen from Table \ref{reg_comparison}, our proposed risk-penalty regularization significantly outperforms ReLU and ABS, which demonstrates the superiority of our proposed risk-penalty regularization over the risk correction scheme.

\begin{figure*}
    \centering
    \subfigure[]{\includegraphics[width=1.5in]{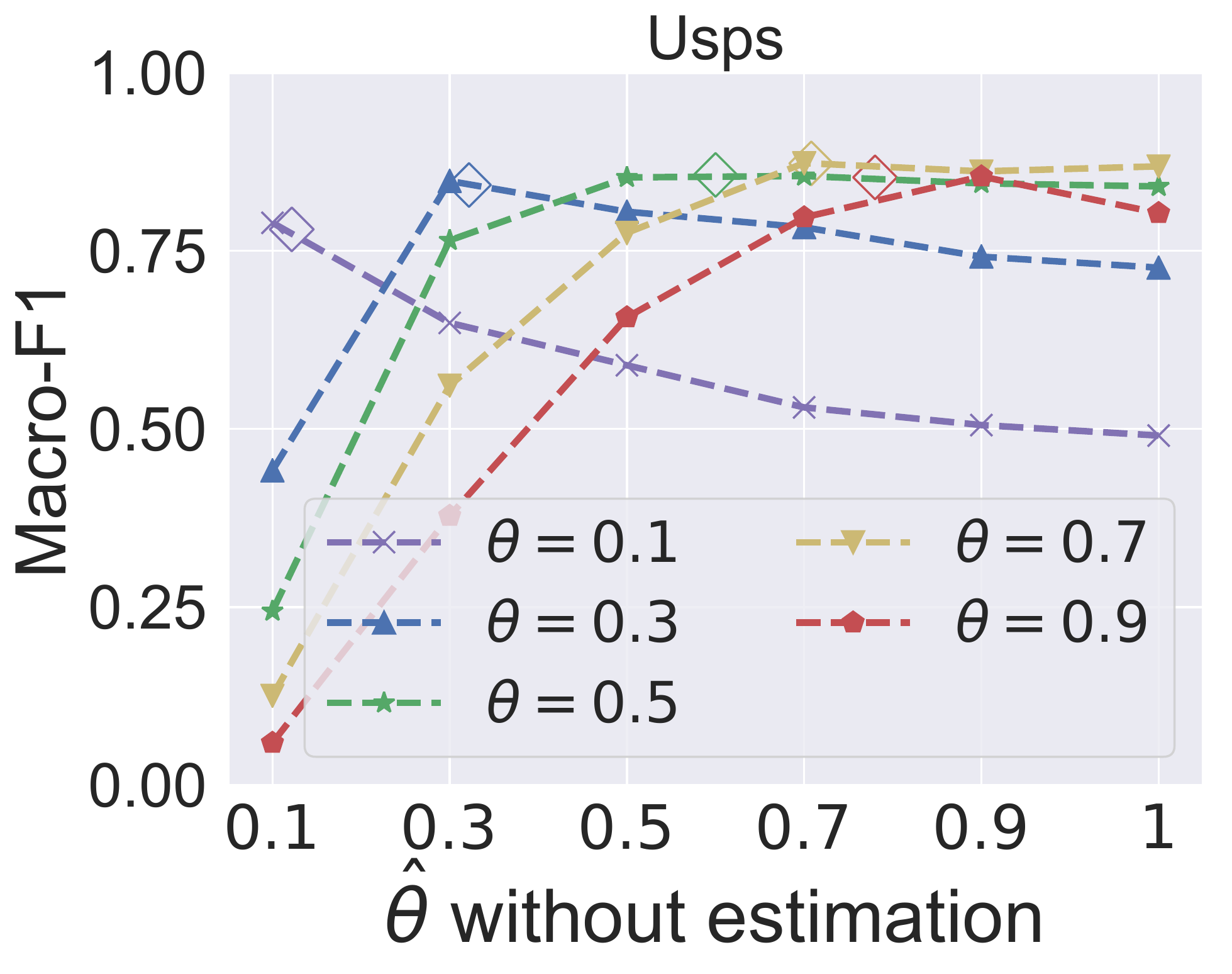}}
    \subfigure[]{\includegraphics[width=1.5in]{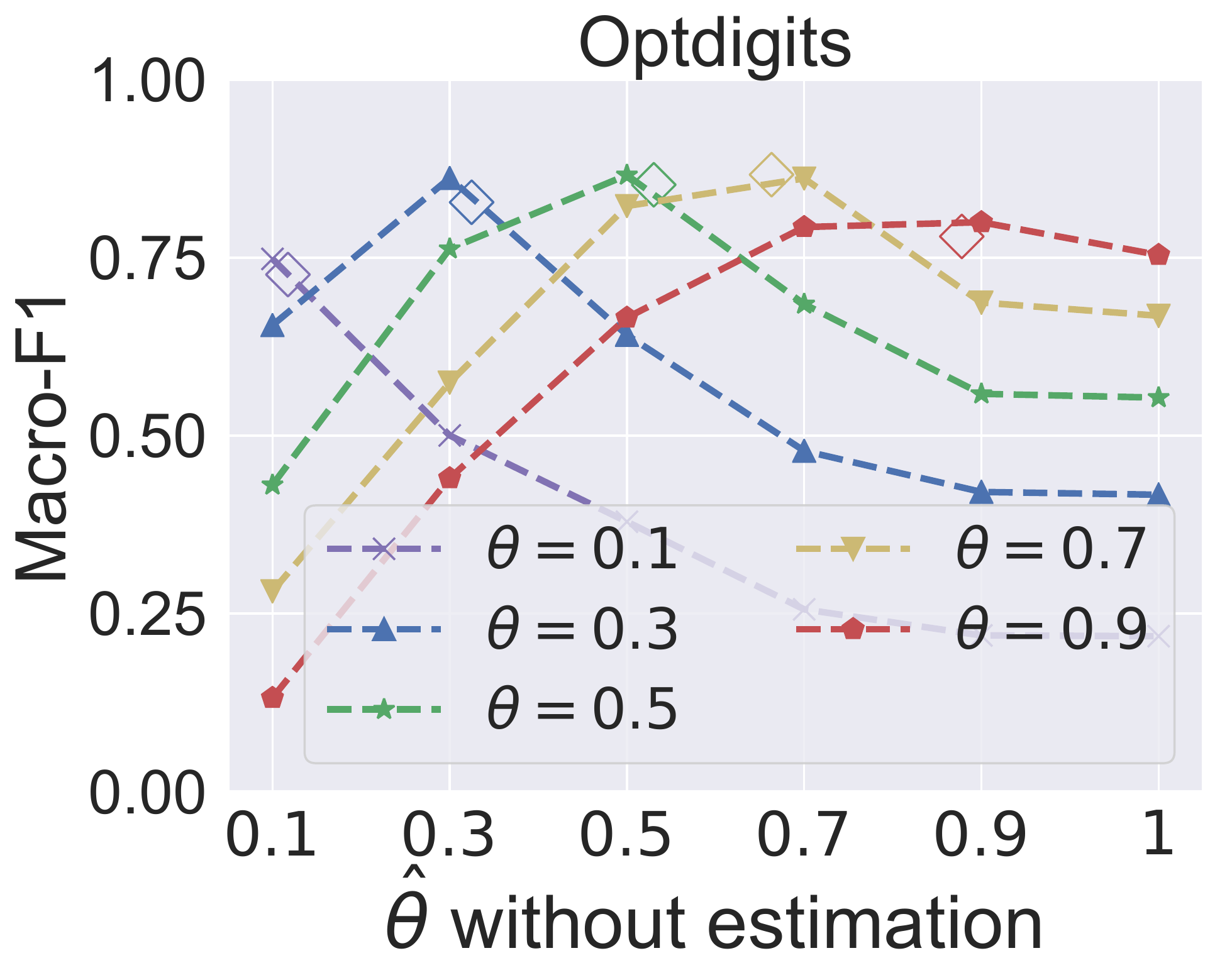}}
     \subfigure[]{\includegraphics[width=1.5in]{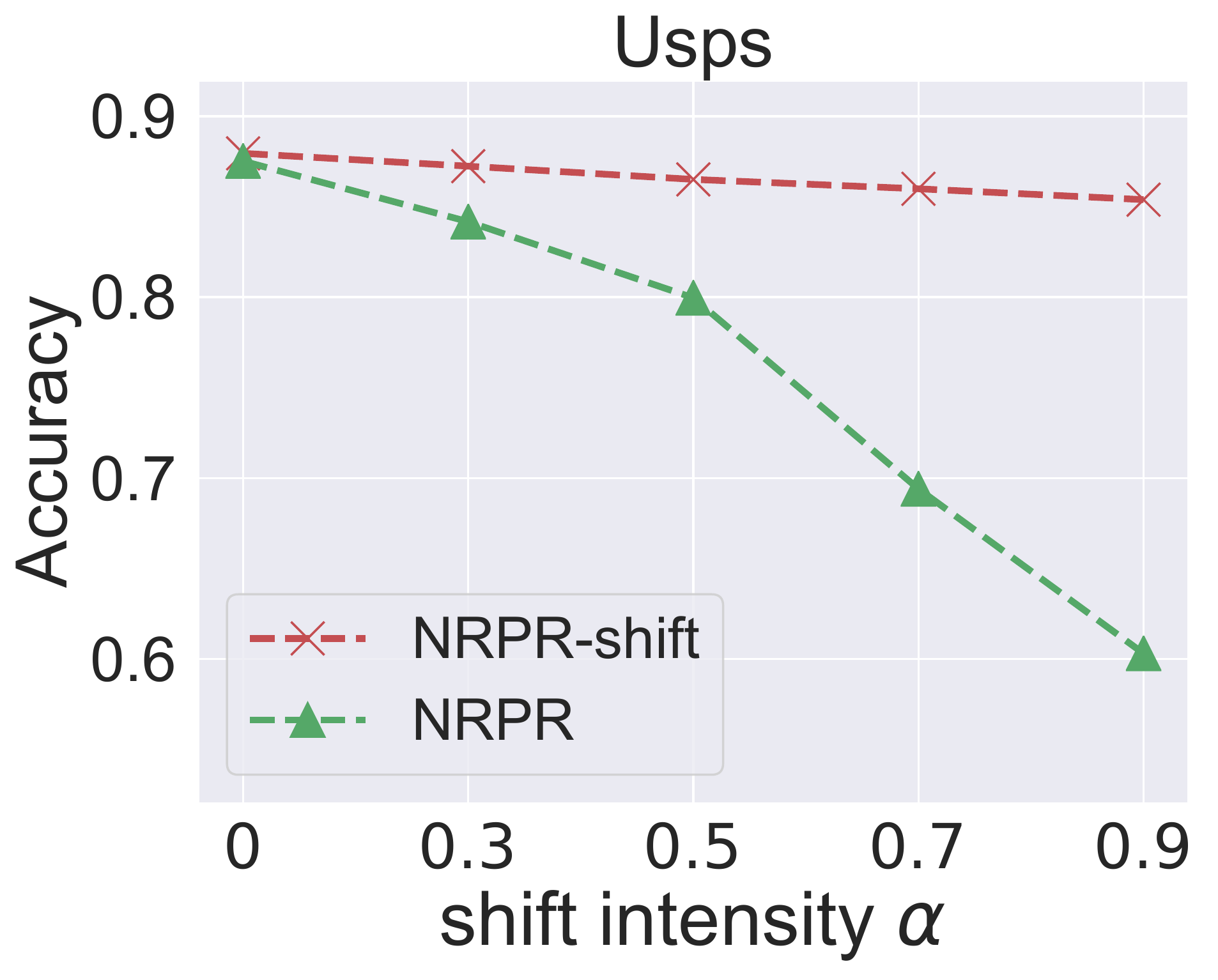}}
    \subfigure[]{\includegraphics[width=1.5in]{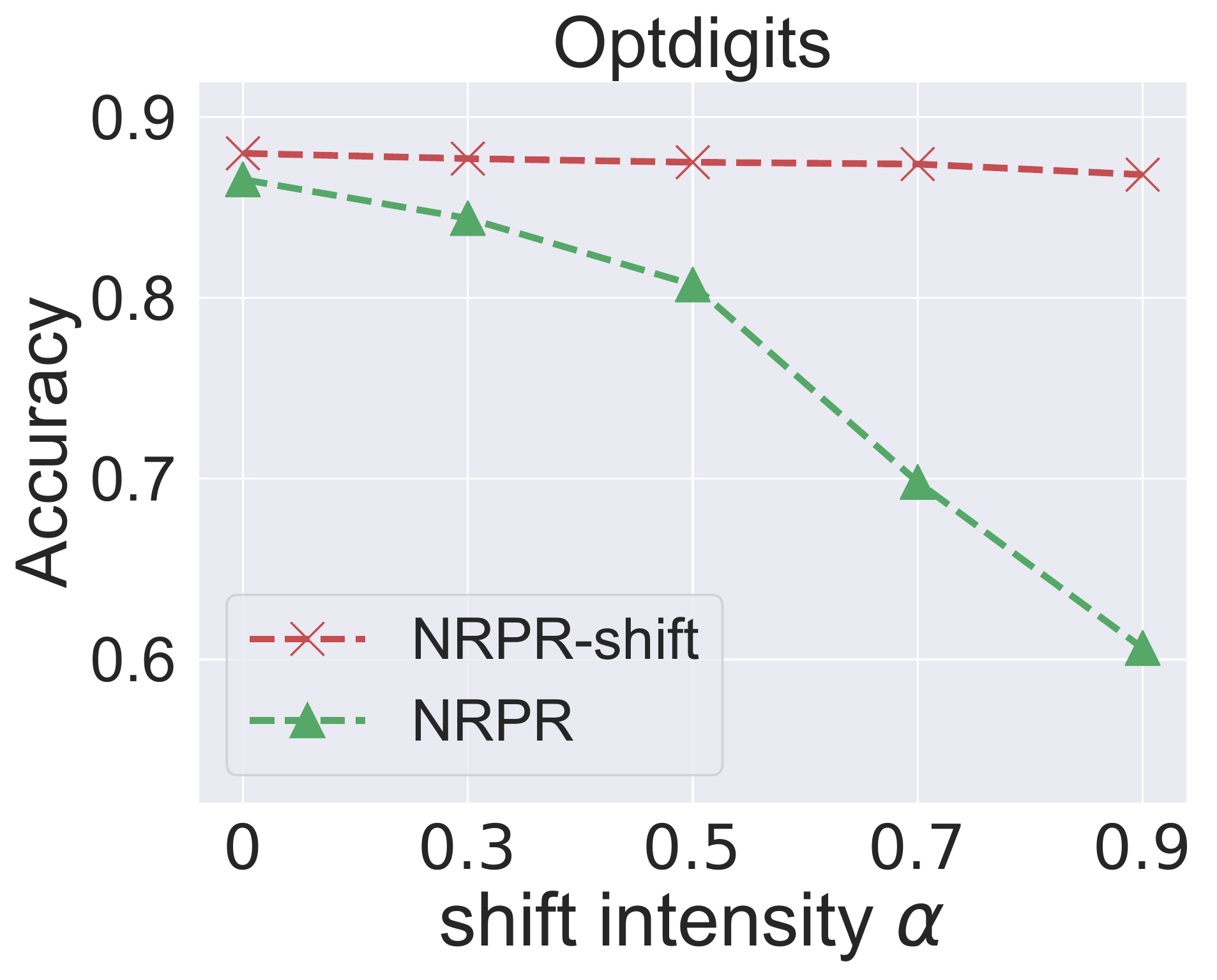}}
       \caption{(a)-(b): Influence of the mixture proportion $\theta$ ($\lozenge$ denotes the empirically estimated mixture proportion $\hat{\theta}$); (c)-(d): Ability to handle class prior shift.}
    \label{class_prior_shift}
\end{figure*}
\begin{figure*}[!thb]
    \centering
    \subfigure{\includegraphics[width=1.5in]{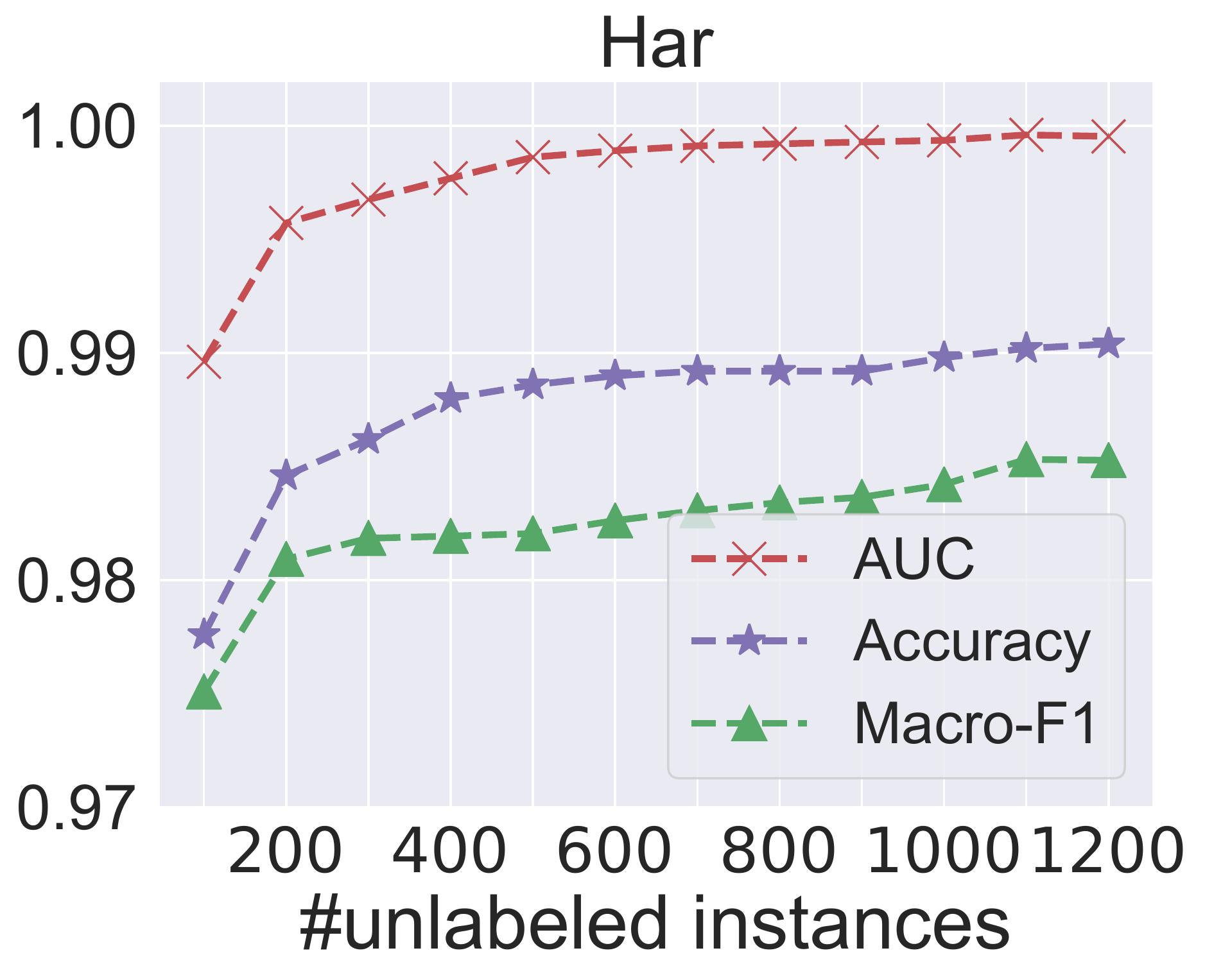}}
    \subfigure{\includegraphics[width=1.5in]{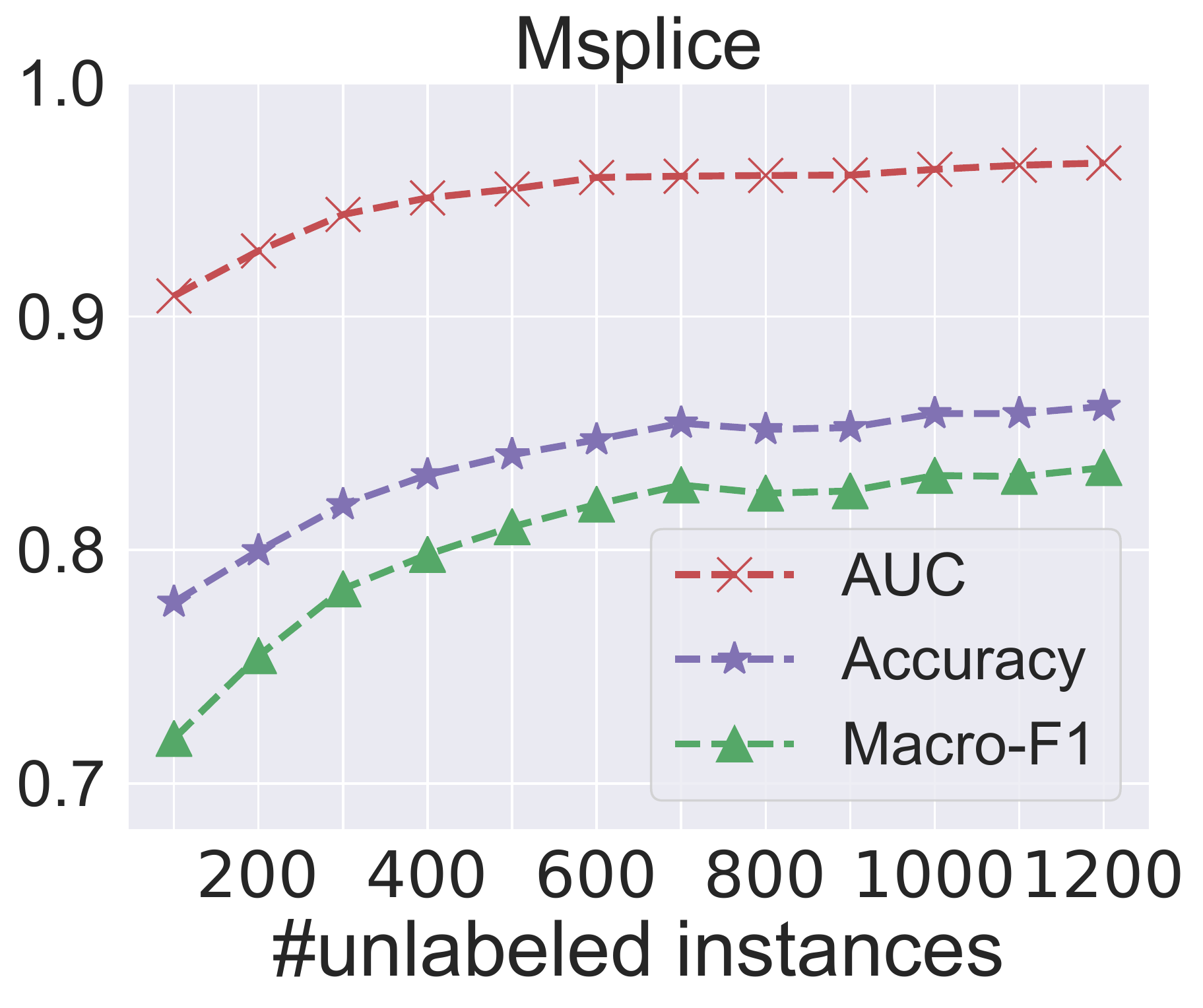}}
    \subfigure{\includegraphics[width=1.5in]{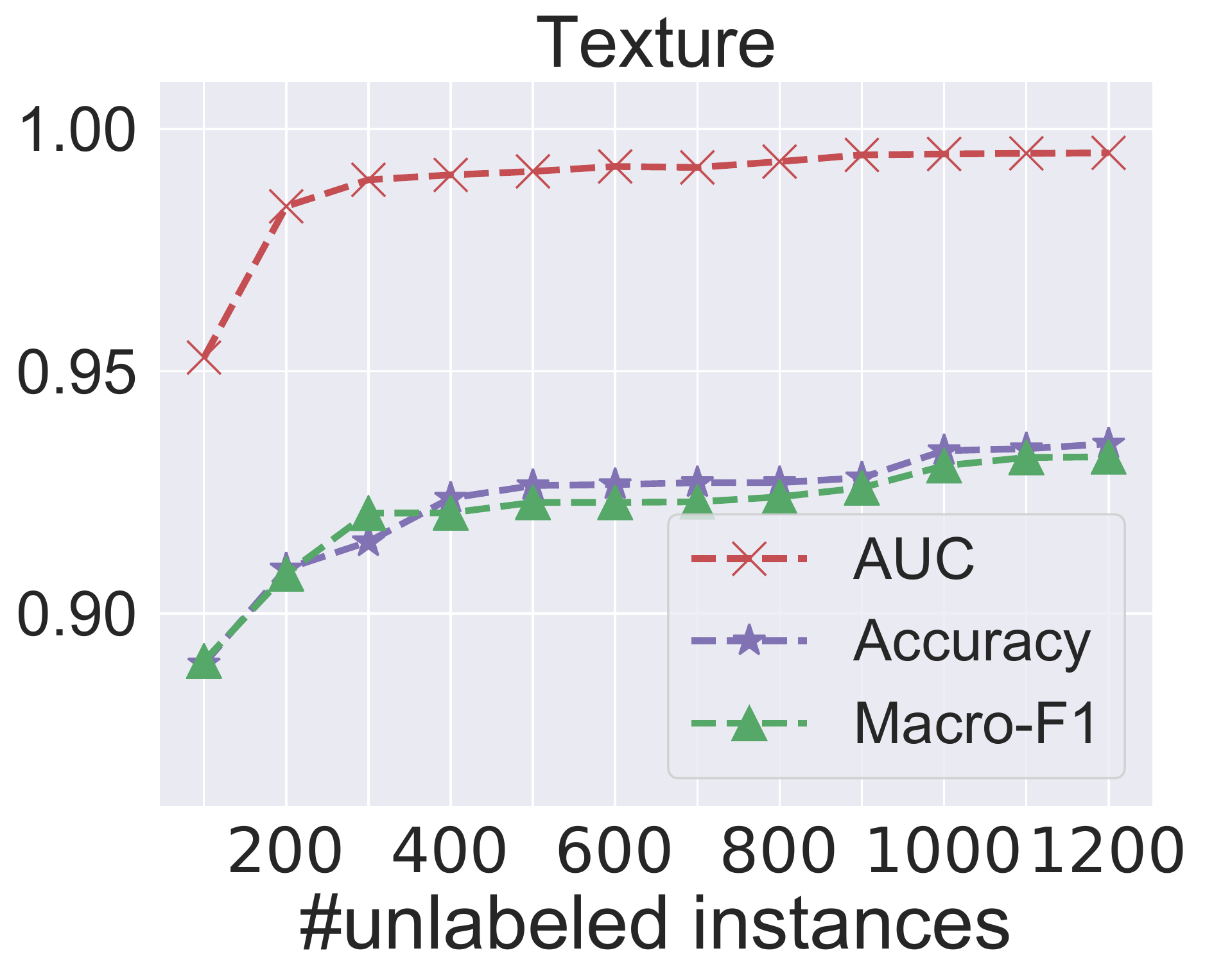}}
    \subfigure{\includegraphics[width=1.5in]{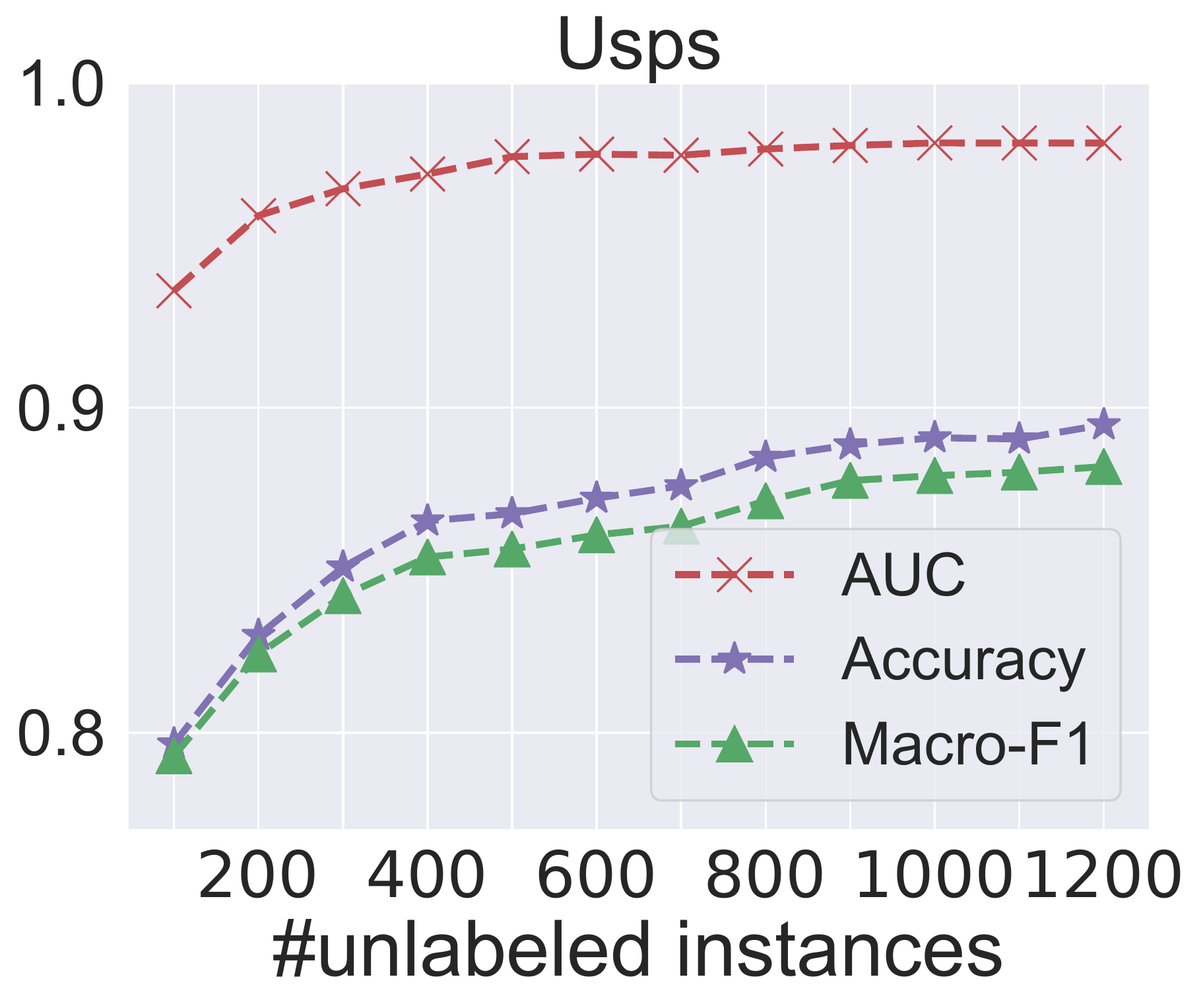}}
    \caption{Test performance on four regular-scale datasets when the number of unlabeled examples increases.}
    \label{vary_numU}
\end{figure*}
\paragraph{Regularization parameter analysis.} To show the concrete effect of our risk-penalty regularization, we further conduct parameter sensitivity analysis of the regularization parameter $\lambda$ on four datasets. Specifically, $\lambda$ is varied in $\{0, 0.2, 0.4, \dots, 1.8, 2\}$ ($\lambda=0$ means that the risk-penalty regularization is not used). From Figure \ref{parameter_analysis}, we can find that the risk-penalty regularization indeed improves the learning performance and the best performance is achieved at some intermediate value of $\lambda$, which suggests that the regularization plays an important role.

\paragraph{Influence of the mixture proportion.} For the mixture parameter $\theta$, we follow Zhang et al. \cite{zhang2020unbiased} to empirically estimate $\theta$ by using the kernel mean embedding method \cite{ramaswamy2016mixture} on regular datasets. To show the influence of the mixture proportion $\theta$, we conduct experiments on the Usps and Optdigits datasets by varying the preseted mixture proportion $\hat{\theta}$ from 0.1 to 1 under different values of the true mixture proportion $\theta$. As shown in Figure \ref{class_prior_shift} (a)-(b), the true mixture proportion can be accurately estimated in most cases.

\paragraph{Handling class prior shift.} In order to verify the ability of our proposed method to handle more complex learning environments, we further conduct experiments on Usps and Optdigits with class prior shift. Specifically, we select $\{1,2,3,4,5\}$ as known classes with prior 0.2 per class in labeled data and $\{6,7,8,9,10\}$ as augmented classes with prior 0.1 per class in test and unlabeled data. As for the prior shift in the test distribution, we reset the prior of five known classes to $\{1-\alpha,1-\frac{\alpha}{2},1,1+\frac{\alpha}{2},1+\frac{\alpha}{2}\}\times 0.1$, where $\alpha$ controls shift intensity and is selected in $\{0,0.3,0.5,0.7,0.9\}$. Figure \ref{class_prior_shift} (c)-(d) reports the accuracy for NRPR (ignoring class prior shift) and NRPR-shift (considering class prior shift in model training according to our derived Theorem \ref{thm3}). As shown in Figure \ref{class_prior_shift} (c)-(d), when the shift intensity increases, performance of NRPR rapidly decreases while NRPR-shift retains satisfactory performance even with a high shift intensity. This observation suggests that our proposed NRPR-shift method could well handle changing learning environments.

\paragraph{Performance of increasing unlabeled examples.} As shown in Theorem \ref{thm2}, the performance of our NRPR method is expected to be improved if more unlabeled examples are provided. To empirically validate such a theoretical finding, we conduct experiments on four regular-scale datasets by changing the number of unlabeled examples from 100 to 1200. As shown in Figure \ref{vary_numU}, the accuracy, Macro-F1, and AUC of our NRPR method generally increases given more unlabeled examples. This observation clearly supports the derived estimation error bound in Theorem \ref{thm2}.

\section{Related Work}
In this section, we discuss some research topics and studies that are relevant to our work.
\paragraph{Class-incremental learning.} As many practical machine learning systems require the ability to continually learn new knowledge, \emph{class-incremental learning} (CIL) \cite{zhou2002hybrid,Tao_2020_CVPR} was proposed to handling the increasing number of classes. \emph{Learning with augmented classes} (LAC) \cite{da2014learning,zhang2020unbiased,ding2018imbalanced}, the problem studied in this paper, is a main task in CIL. The previous research on LAC \cite{da2014learning} showed that unlabeled data can be exploited to enhance the performance of LAC. Furthermore, Zhang et al. \cite{zhang2020unbiased} showed that given unlabeled data, an unbiased risk estimator (URE) can be derived, which can be minimized for LAC with theoretical guarantees. 
\paragraph{Open-set recognition.} In \emph{open-set recognition} (OSR) \cite{geng2020recent}, new classes unseen in training the phase could appear in the test phase, and the learned classifier is required to not only accurately classify known classes but also effectively deal with unknown classes. Therefore, OSR can be seen as a cousin of the LAC problem in the computer vision and pattern recognition communities. Most studies on OSR aim to explore the relationship between known and augmented classes by using specific techniques, such as the open space risk \cite{scheirer2014probability}, nearest neighbor approach \cite{mendes2017nearest}, and extreme value theory \cite{rudd2017extreme}. It is noteworthy that many OSR methods utilized the domain knowledge, while our work focuses on a general setting without additional information.

\paragraph{Multi-class positive-unlabeled learning.} Learning from multi-class positive-unlabeled data \cite{xu2017multi,shu2020learning} is similar to the LAC problem with unlabeled data, by taking the known classes as the multi-class positive. Although our work shares similarity with the two studies \cite{xu2017multi,shu2020learning}, our LAC risk derived in a different context brings novel understandings for the LAC and can handle more complex environments (i.e., class prior shift), while they assumed the class priors to be known.

\section{Conclusion}
In this paper, we studied the problem of \emph{learning with augmented classes} (LAC). To solve this problem, we proposed a generalized \emph{unbiased risk estimator} (URE) that can be equipped with arbitrary loss functions while maintaining the theoretical guarantees, given unlabeled data for LAC. We showed that our generalized URE can recover the URE proposed by the previous study \cite{zhang2020unbiased}. We also established an estimation error bound, which achieves the optimal parametric convergence rate. To alleviate the issue of negative empirical risk commonly encountered by previous studies, we further proposed a novel risk-penalty regularization term. Comprehensive experimental results clearly verified the effectiveness of our proposed method. In future work, it would be interesting to develop more powerful regularization techniques for further improving the performance of our proposed method. 
\section{Acknowledgments}
This work is supported by the National Natural Science Foundation of China (Grant No. 62106028), Chongqing Overseas Chinese Entrepreneurship and Innovation Support Program, and CAAI-Huawei MindSpore Open Fund.

\bibliography{aaai23}
\newpage
\appendix
\section{Proof of Theorem 1}\label{proof_theorem_1}

Recall that
\begin{align}
\nonumber
R(f)=\theta\mathbb{E}_{(\bm{x},y)\sim P_{\mathrm{kc}}}[\mathcal{L}(f(\bm{x}),y)] + (1-\theta)\mathbb{E}_{(\bm{x},y=\mathtt{ac})\sim P_{\mathrm{ac}}}[\mathcal{L}(f(\bm{x}),\mathtt{ac})],
\end{align}
we can observe that there are two parts in the classification risk $R(f)$, including the risk of know classes $R_{\mathrm{kc}}=\mathbb{E}_{(\bm{x},y)\sim P_{\mathrm{kc}}}[\mathcal{L}(f(\bm{x}),y)]$ and the risk of the augmented class $R_{\mathrm{ac}}=\mathbb{E}_{(\bm{x},y=\mathtt{ac})\sim P_{\mathrm{ac}}}[\mathcal{L}(f(\bm{x}),\mathtt{ac})]$. Then $R(f)$ can be represented by $R(f)=\theta R_{\mathrm{kc}}(f)+(1-\theta)R_{\mathrm{ac}}(f)$.

According to 
$P_{\mathrm{te}} = \theta\cdot P_{\mathrm{kc}} + (1-\theta)\cdot P_{\mathrm{ac}}$, we have $(1-\theta)P_{\mathrm{ac}} = P_{\mathrm{te}} - \theta P_{\mathrm{kc}}$. Thus we can express $(1-\theta)R_{\mathrm{ac}}(f)$ by 
\begin{align}
\nonumber
&(1-\theta)\mathbb{E}_{(\bm{x},y=\mathtt{ac})\sim P_{\mathrm{ac}}}[\mathcal{L}(f(\bm{x}),\mathtt{ac})]=\mathbb{E}_{(\bm{x},y)\sim P_{\mathrm{te}}}[\mathcal{L}(f(\bm{x}),\mathtt{ac})]-\theta\mathbb{E}_{(\bm{x},y)\sim P_{\mathrm{kc}}}[\mathcal{L}(f(\bm{x}),\mathtt{ac})].
\end{align}
By substituting the above equality into $R(f)$, we have
\begin{align}
\nonumber
R(f)&=\theta\mathbb{E}_{(\bm{x},y)\sim P_{\mathrm{kc}}}[\mathcal{L}(f(\bm{x}),y)]+\mathbb{E}_{(\bm{x},y)\sim P_{\mathrm{te}}}[\mathcal{L}(f(\bm{x}),\mathtt{ac})]-\theta\mathbb{E}_{(\bm{x},y)\sim P_{\mathrm{kc}}}[\mathcal{L}(f(\bm{x}),\mathtt{ac})]\\
\nonumber
&=\theta \mathbb{E}_{(\bm{x},y)\sim P_{\mathrm{kc}}}[\mathcal{L}(f(\boldsymbol{x}),y)-\mathcal{L}(f(\boldsymbol{x}),\mathtt{ac})] + \mathbb{E}_{\bm{x}\sim P_{\mathrm{te}}}[\mathcal{L}(f(\boldsymbol{x}),\mathtt{ac})]\\
\nonumber
&=R_{\mathrm{LAC}}(f),
\end{align}
which completes the proof of Theorem 1.\qed

\section{Proof of Theorem 2}\label{proof_theorem_2}
Recall that the unbiased risk estimator we derived is represented as follows:
\begin{align}
\nonumber
\widehat{R}_{\mathrm{LAC}}(f) 
&= \frac{\theta}{n}\sum\nolimits_{i=1}^n\Big(\mathcal{L}(f(\boldsymbol{x}_i),y_i)-\mathcal{L}(f(\boldsymbol{x}_i),\mathtt{ac})\Big) + \frac{1}{m}\sum\nolimits_{j=1}^m\mathcal{L}(f(\boldsymbol{x}_j),\mathtt{ac}).
\end{align}
Let us further introduce
\begin{align}
\nonumber
\widehat{R}_{\mathrm{kac}}(f) &= \frac{\theta}{n}\sum\nolimits_{i=1}^n\Big(\mathcal{L}(f(\boldsymbol{x}_i),y_i)-\mathcal{L}(f(\boldsymbol{x}_i),\mathtt{ac})\Big),\\
\nonumber
\widehat{R}_{\mathrm{tac}}(f) &= \frac{1}{m}\sum\nolimits_{j=1}^m\mathcal{L}(f(\boldsymbol{x}_j),\mathtt{ac}),\\
\nonumber
{R}_{\mathrm{kac}}(f) &= \mathbb{E}_{(\bm{x},y)\sim P_{\mathrm{kc}}}\big[\mathcal{L}(f(\boldsymbol{x}),y)-\mathcal{L}(f(\boldsymbol{x}),\mathtt{ac})\big],\\
\nonumber
{R}_{\mathrm{tac}}(f) &= \mathbb{E}_{\bm{x}\sim P_{\mathrm{te}}}\big[\mathcal{L}(f(\boldsymbol{x}),\mathtt{ac})\big].
\end{align}
Therefore, we have
\begin{align}
\nonumber
\widehat{R}_{\mathrm{LAC}}(f) &= \theta\cdot \widehat{R}_{\mathrm{kac}}(f) + \widehat{R}_{\mathrm{tac}},\\
\nonumber
{R}_{\mathrm{LAC}}(f) &= \theta\cdot {R}_{\mathrm{kac}}(f) + {R}_{\mathrm{tac}}.
\end{align}
Thus we obtain
\begin{align}
\nonumber
\sup_{f\in\mathcal{F}}\left|{R}_{\mathrm{LAC}}(f)-\widehat{R}_{\mathrm{LAC}}(f)\right|&\leq \theta\cdot\sup_{f\in\mathcal{F}}\left|{R}_{\mathrm{kac}}(f)-\widehat{R}_{\mathrm{kac}}(f)\right|+\sup_{f\in\mathcal{F}}\left|{R}_{\mathrm{tac}}(f)-\widehat{R}_{\mathrm{tac}}(f)\right|.
\end{align}
Hence the problem becomes how to find an upper bound of each term in the right hand side of the above inequality.
\begin{lemma}
\label{lem_kac}
Assume the multi-class loss function $\mathcal{L}(f(\bm{x}),y)$ is $\rho$-Lipschitz ($0<\rho<\infty$) with respect to $f(\bm{x})$ for all $y\in\mathcal{Y}$ and upper bounded by a constant $C_{\mathcal{L}}$, i.e., $C_{\mathcal{L}} = \sup_{\bm{x}\in\mathcal{X},y\in\mathcal{Y},f\in\mathcal{F}}\mathcal{L}(f(\bm{x}),y)$. Then, for any $\delta>0$, with probability at least $1-\delta$, we have
\begin{align}
\nonumber
\sup_{f\in\mathcal{F}}\left|{R}_{\mathrm{kac}}(f)\!-\!\widehat{R}_{\mathrm{kac}}(f)\right| \!\leq\! 4\sqrt{2}\rho(k+1)\frac{C_{\mathcal{F}}}{n} \!+\! 2C_{\mathcal{L}}\sqrt{\frac{\log\frac{2}{\delta}}{2n}}.
\end{align}
\end{lemma}
\begin{proof}
Suppose an example in $\widehat{R}_{\mathrm{kac}}(f)$ is replaced by another arbitrary example, then the change of $\sup_{f\in\mathcal{F}}\big({R}_{\mathrm{kac}}(f) - \widehat{R}_{\mathrm{kac}}(f)\big)$ is no greater than $\frac{2C_{\mathcal{L}}}{n}$. Then, by applying the Diarmid's inequality \cite{mcdiarmid1989method}, for any $\delta>0$, with probability at least $1-\frac{\delta}{2}$,
\begin{align}
\nonumber
\sup\nolimits_{f\in\mathcal{F}}\big({R}_{\mathrm{kac}}(f) - \widehat{R}_{\mathrm{kac}}(f)\big) &\leq 
\mathbb{E}\Big[\sup\nolimits_{f\in\mathcal{F}}\big({R}_{\mathrm{kac}}(f) - \widehat{R}_{\mathrm{kac}}(f)\big)\Big] + 2C_{\mathcal{L}}\sqrt{\frac{\log\frac{2}{\delta}}{2n}}.
\end{align}
Besides, it is routine \cite{mohri2012foundations} to show
\begin{gather}
\nonumber
\mathbb{E}\Big[\sup\nolimits_{f\in\mathcal{F}}\big({R}_{\mathrm{kac}}(f) - \widehat{R}_{\mathrm{kac}}(f)\big)\Big] \leq 4\mathfrak{R}_n(\mathcal{L}\circ\mathcal{F}).
\end{gather}
Then, we need to upper bound ${\mathfrak{R}}_n(\mathcal{L}\circ\mathcal{F})$. Since $\mathcal{L}$ is $\rho$-Lipschitz with respect to $f(\boldsymbol{x})$, according to the \emph{Rademacher vector contraction inequality} \cite{maurer2016vector}, we have
\begin{gather}
\nonumber
\widehat{\mathfrak{R}}_n(\mathcal{L}\circ\mathcal{F})\leq \sqrt{2}\rho\sum\nolimits_{y=1}^{k+1}\widehat{\mathfrak{R}}_n(\mathcal{F}_{y}).
\end{gather}
By taking the expectation of $\widehat{\mathfrak{R}}_n(\mathcal{L}\circ\mathcal{F})$ and $\widehat{\mathfrak{R}}_n(\mathcal{F}_{y})$ over $p(\bm{x})$, we have 
${\mathfrak{R}}_n(\mathcal{L}\circ\mathcal{F})\leq \sqrt{2}\rho\sum_{y=1}^{k+1}{\mathfrak{R}}_n(\mathcal{F}_{y})$. By further considering $\mathfrak{R}_n(\mathcal{F}_y)\leq C_{\mathcal{F}}/\sqrt{n}$, then we have for any $\delta>0$, with probability at least $1-\frac{\delta}{2}$,
\begin{align}
\nonumber
\sup_{f\in\mathcal{F}}\!\big({R}_{\mathrm{kac}}(f) \!-\! \widehat{R}_{\mathrm{kac}}(f)\big) \!\leq\! 
4\sqrt{2}\rho(k+1)\frac{C_{\mathcal{F}}}{\sqrt{n}} \!+\! 2C_{\mathcal{L}}\sqrt{\frac{\log\frac{2}{\delta}}{2n}}.
\end{align}
By further considering the other side $\sup_{f\in\mathcal{F}}\big(\widehat{R}_{\mathrm{kac}}(f)-{R}_{\mathrm{kac}}(f)\big)$, we have for $\delta>0$, with probability at least $1-\delta$,
\begin{align}
\nonumber
\sup_{f\in\mathcal{F}}\left|{R}_{\mathrm{kac}}(f) \!-\! \widehat{R}_{\mathrm{kac}}(f)\right| \!\leq\! \
4\sqrt{2}\rho(k\!+\!1)\frac{C_{\mathcal{F}}}{\sqrt{n}} \!+\! 2C_{\mathcal{L}}\sqrt{\frac{\log\frac{2}{\delta}}{2n}},
\end{align}
which completes the proof of Lemma \ref{lem_kac}.
\end{proof}

\begin{lemma}
\label{lem_tac}
Assume the multi-class loss function $\mathcal{L}(f(\bm{x}),y)$ is $\rho$-Lipschitz ($0<\rho<\infty$) with respect to $f(\bm{x})$ for all $y\in\mathcal{Y}$ and upper bounded by a constant $C_{\mathcal{L}}$, i.e., $C_{\mathcal{L}} = \sup_{\bm{x}\in\mathcal{X},y\in\mathcal{Y},f\in\mathcal{F}}\mathcal{L}(f(\bm{x}),y)$. Then, for any $\delta>0$, with probability at least $1-\delta$, we have
\begin{align}
\nonumber
\sup_{f\in\mathcal{F}}\left|{R}_{\mathrm{tac}}(f)-\widehat{R}_{\mathrm{tac}}(f)\right| \!\leq\! 2\sqrt{2}\rho(k+1)\frac{C_{\mathcal{F}}}{\sqrt{m}} + C_{\mathcal{L}}\sqrt{\frac{\log\frac{2}{\delta}}{2m}}.
\end{align}
\end{lemma}
\begin{proof}
Lemma \ref{lem_tac} can be proved similarly as Lemma \ref{lem_kac}, hence we omit the proof.
\end{proof}
By combining Lemma \ref{lem_kac} and Lemma \ref{lem_tac} together, for $\delta>0$, with probability at least $1-\delta$, we have
\begin{align}
\nonumber
\sup\nolimits_{f\in\mathcal{F}}\left|R_{\mathrm{LAC}}(f)-\widehat{R}_{\mathrm{LAC}}(f)\right| &\leq C_{k,\rho,\delta}(\frac{2\theta}{\sqrt{n}}+\frac{1}{\sqrt{m}}),
\end{align}
where $C_{k,\rho,\delta} = (2\sqrt{2}\rho(k+1)C_{\mathcal{F}}+C_{\mathcal{L}}\sqrt{\frac{\log\frac{4}{\delta}}{2}})$.

In addition, we can obtain
\begin{align}
\nonumber
R(\widehat{f})-R(f^\star)
&=R_{\mathrm{LAC}}(\widehat{f})-R_{\mathrm{LAC}}(f^\star)\\
\nonumber
&= R_{\mathrm{LAC}}(\widehat{f}) - \widehat{R}_{\mathrm{LAC}}(\widehat{f}) + \widehat{R}_{\mathrm{LAC}}(\widehat{f}) - R_{\mathrm{LAC}}({f}^\star)\\
\nonumber
&\leq R_{\mathrm{LAC}}(\widehat{f}) - \widehat{R}_{\mathrm{LAC}}(\widehat{f}) + R_{\mathrm{LAC}}(\widehat{f}) - R_{\mathrm{LAC}}({f}^\star)\\
\nonumber
&\leq 2\sup\nolimits_{f\in\mathcal{F}}\left|R_{\mathrm{LAC}}(f)-\widehat{R}_{\mathrm{LAC}}(f)\right|.
\end{align}
Therefore, for $\delta>0$, with probability at least $1-\delta$, we have
\begin{align}
\nonumber
R(\widehat{f})-R(f^\star) &\leq C_{k,\rho,\delta}(\frac{2\theta}{\sqrt{n}}+\frac{1}{\sqrt{m}}),
\end{align}
where $C_{k,\rho,\delta} = (4\sqrt{2}\rho(k+1)C_{\mathcal{F}}+2C_{\mathcal{L}}\sqrt{\frac{\log\frac{4}{\delta}}{2}})$. \qed

\bibliographystyle{ieeetr}
\end{document}